\documentclass[10pt]{article} % For LaTeX2e

\usepackage[preprint]{tmlr}

\usepackage{amsmath}
\usepackage{amssymb}
\usepackage{graphicx}
\usepackage[utf8]{inputenc}    % allow utf-8 input
\usepackage[T1]{fontenc}       % use 8-bit T1 fonts
\usepackage{hyperref}          % hyperlinks
\usepackage{url}               % simple URL typesetting
\usepackage{booktabs}          % professional-quality tables
\usepackage{amsfonts}          % blackboard math symbols
\usepackage{nicefrac}          % compact symbols for 1/2, etc.
\usepackage{microtype}         % microtypography
\usepackage[ruled,linesnumbered]{algorithm2e}
\usepackage{multirow} 

\usepackage{amsthm}
\newtheorem{theorem}{Theorem}[section]
\newtheorem{proposition}[theorem]{Proposition}
\newtheorem{lemma}[theorem]{Lemma}
\newtheorem{corollary}[theorem]{Corollary}
\theoremstyle{definition}
\newtheorem{definition}[theorem]{Definition}
\newtheorem{example}[theorem]{Example}

\theoremstyle{remark}
\newtheorem{remark}[theorem]{Remark}

\title{Exact Algebraic Computation of Learning Coefficients for Two-Dimensional Singular Models}

\author{\name Gregoire Sergeant-Perthuis \email gregoire.sergeant-perthuis@sorbonne-universite.fr \\
      \addr CQSB, Sorbonne Universit\'e
      \AND
      \name Elias Tsigaridas \email elias.tsigaridas@inria.fr \\
      \addr OURAGAN, Inria Paris
      \AND
      \name Jules Tsukahara \email jules.tsukahara@inria.fr \\
      \addr OURAGAN, Inria Paris}

\usepackage{subcaption}
\usepackage{mathtools}
\usepackage{float}
\usepackage[capitalize,noabbrev]{cleveref}
\usepackage{comment}
\usepackage{xcolor}
\usepackage{bm}
\usepackage{tikz-cd}
\usepackage{units}
\usepackage{enumitem}
\newenvironment{sproof}{%
  \proof}{\endproof}

\newcommand{\RR}{\mathbb{R}}

\newcommand{\QQ}{\mathbb{Q}}
\newcommand{\KK}{\mathbb{K}}
\newcommand{\ZZ}{\mathbb{Z}}
\newcommand{\NN}{\mathbb{N}}
\newcommand{\newton}{\mathcal{N}}

\newcommand{\puiseux}[2]{\mathbb{#1}\langle\langle#2\rangle\rangle}
\newcommand{\trunc}[2]{\lceil #1 \rceil^#2}
\newcommand{\finitepart}[2]{P^{\text{fin}}_#1(#2)}
\newcommand{\abs}[1]{|#1|}

\newcommand{\rlct}{\operatorname{RLCT}}
\newcommand{\ord}{\operatorname{ord}}
\newcommand{\res}{\operatorname{res}}
\newcommand{\disc}{\operatorname{disc}}

\newcommand{\Syl}{\operatorname{Syl}}
\newcommand{\Conv}{\operatorname{Conv}}

\usepackage[normalem]{ulem}

\SetArgSty{upshape}
\SetAlCapHSkip{0pt}
\DontPrintSemicolon

\def\month{MM}  % 2-digit month of publication
\def\year{YYYY} % 4-digit year of publication
\def\openreview{\url{https://openreview.net/forum?id=XXXX}}

\begin{document}

\maketitle

\begin{abstract}%
 Classical information criteria such as the Bayesian Information Criterion (BIC) rely on regularity assumptions that break down for singular models, leading to incorrect model selection in settings such as deep learning. The Widely Applicable Bayesian Information Criterion (WBIC) relies on  local learning coefficients $\lambda$, which in the analytic case coincides with local Real Log Canonical Thresholds (RLCT) of the Kullback-Leibler divergence of the model, to capture correct marginal likelihood asymptotics. Exact computation of the learning coefficients has been limited to special cases, and only sampling-based estimation methods are generally applicable. We present the first deterministic algorithm that  computes local RLCTs exactly for any two-dimensional model whose Kullback-Leibler distance is contact equivalent to a polynomial, derive a bound on its complexity, and demonstrate its effectiveness for a broad class of models, with applications including polynomial neural networks. Beyond providing ground truth to calibrate sampling-based estimators, exact computation reveals algebraic structure in learning coefficients that sampling cannot and out-speeds it in the shallow regime.
\end{abstract}

\newpage

\tableofcontents

\newpage

\section{Introduction}

\textbf{Information criteria (IC):} IC are model selection methods in statistics and machine learning \citep{konishi2008information}. They aim to strike a balance between prediction accuracy and model complexity. They are ubiquitous in diverse domains, ranging from statistics~\citep{zhang2023information}, especially in the presence of structured learning~\citep{FoygelBarberDrton2015,de2011efficient}, and dimensionality reduction~\citep{sclove_using_2021,tomarchio2025number},
to biology~\citep{susko_use_2020}, neuroscience~\citep{penny2012comparing},
and dynamical systems~\citep{thanasutives_adaptive_2024}. A statistical model is said to be regular if its parameter-to-model map is injective (the identifiability condition) and its Fisher Information Metric is positive definite everywhere on the parameter space (the positive-definiteness condition)~\cite{liu2025learning}. For regular models, the celebrated Bayesian information
criterion (BIC) asymptotically approximates the log marginal likelihood and thus carries statistical meaning for posterior inference. Through the Laplace approximation \citep{schwarz1978estimating}, the BIC has an explicit formula we can compute directly from the input~data.

\textbf{Singular models:} Regularity assumptions are generically violated in deep learning due to the overparameterization of modern architectures \citep{wei2022deep, liu2025learning}.
Hence, for non-regular models, the various statistical quantities are estimated using approximate and costly methods such as sampling \citep{liu2025learning,newman2026fast,lau2025the}. The real log canonical threshold (RLCT) is a key quantity from algebraic geometry that characterizes the asymptotic behavior of the log marginal likelihood for non-regular models, generalizing the role of model dimension in the BIC \citep{watanabe2009algebraic}. Unlike the explicit BIC formula, computing the RLCT requires resolving singularities in the parameter space; a very challenging computational and mathematical problem. 
For recent progress in applications of singular learning theory to information theoretic compression we refer to \cite{ulhwm-ml-slt-26}.

\textbf{Polynomial neural networks and identifiability:} A class of neural networks whose activation functions are monomials $\sigma(x) = x^r$, coined \emph{polynomial neural networks} (PNNs), are gaining interest for their theoretical properties. Their appeal stems from the fact that a PNN defines an \emph{algebraic} map from weights to functions: the composition of monomial activations and linear layers is a polynomial, so the full machinery of algebraic geometry applies. This structure makes it possible to characterize expressivity via algebraic dimension \citep{kileel2019expressive} and degree \citep{kubjas2024geometry}, and gives rise to a rich theory of identifiability \citep{usevich2025identifiability,shahverdi2025learning}. In particular, the asymptotic behavior of the log marginal likelihood, and hence the identifiability of the model under Bayesian inference, is entirely governed by the RLCT \citep{watanabe2009algebraic}.

\textbf{Our Contribution:} We propose, to our knowledge, the first deterministic algorithm that exactly computes the RLCT for any polynomial of two variables, and derive an upper bound on its arithmetic complexity (Prop.~\ref{prop:unique}, Thm.~\ref{thm:finite_part_main}, and Thm.~\ref{thm:complexity_main}). As an application, we compute exact RLCT for polynomial neural networks (PNNs) with repeated weights of increasing depth. We show that the effective model complexity, as measured by the RLCT, can counter-intuitively decrease with the number of layers, implying that already in the two-dimensional case, this gives rise to a non-trivial theory of identifiability. We advocate that such exact algorithms are new tools for the study of loss landscapes in learning theory. An implementation of Algorithm~\ref{alg:RLCT_main} in SageMath is available at
\url{https://github.com/jktsukahara/bivariate-rlct}.

\paragraph{Notation}

Let $\KK$ be a field of characteristic zero. We let $\bar{\KK}$ be the algebraic closure of $\KK$. We denote by $\KK[x,y]$ the ring of polynomials in two variables. We let $\KK[[x]]$ be the ring of formal power series, $\KK((x))$ the field of Laurent series and $\puiseux{K}{x}$ the field of Puiseux series with coefficients in $\KK$. For an integer $n$ and a subset $S \subset \mathbb{R}^n$, we let $\Conv(S)$ be the convex hull of $S$. For any two subsets $A\subseteq\mathbb{R}^n$ and $B\subseteq\mathbb{R}^n$, we let $A\oplus B \coloneq \{ a+b \mid  a\in A, \,b\in B\}$ be their Minkowski sum. For any $n\in\NN_{>0}$, we let $[n] = \{1,\ldots,n\}$. We use bold roman lowercase letters for vectors in $\RR^2$. 

\subsection{Background}

\paragraph{Information Criteria for Model Selection}

Information criteria are model selection methods in statistics and machine learning that balance predictive accuracy against model complexity. A celebrated example of such a criterion is the Bayesian information criterion (BIC) \citep{schwarz1978estimating}, which emanates from the asymptotic behavior of log marginal likelihoods under regularity assumptions on the models. Concretely, consider a family of parametric probability distributions $P(X_1, \dots, X_N \mid \theta, M)$, one for each parameter value \(\theta\in \Theta_M\), with possible values in a Borel subset $\Theta_M \subseteq \mathbb{R}^{d_M}$; each family of model $M$ therefore provides the likelihood of a collection of $N$ samples  \(X_1, \dots, X_N\). For a given prior \(\pi_M(\theta)\) over the parameter space $\Theta_M$,  the marginal likelihood of the data under model \(M\) is  
$
P(x_1, \dots, x_N \mid M) = \int P(x_1, \dots, x_N, \mid\theta,  M)\,\pi_M(\theta)\,\mathrm{d}\theta$. Under appropriate regularity assumptions \citep{schwarz1978estimating,konishi2008information}, we can approximate that integral, for large sample size \(n\), as follows:  
\begin{align}\label{intro:usual-BIC}
-\log P(x_1, \dots, x_N \mid M) \;\approx\; - \log P(x_1, \dots, x_N \mid \hat\theta_M, M)+ \frac{d_M}{2} \log N, \tag{BIC}
\end{align} 
where \(\hat\theta_M\) is the maximum-likelihood estimator (MLE) under model \(M\). The right hand side of Eq.~\ref{intro:usual-BIC} is the expression of the celebrated Bayesian Information Criterion (BIC); it penalizes model complexity (larger \(d_M\)) while rewarding goodness of fit. Several other criteria that balance model fit and model complexity are available, such as the Akaike Information Criterion (AIC), the Generalized Information Criterion (GIC), and the Takeuchi Information Criterion (TIC), as well as Bayesian extensions like Akaike's Bayesian Information Criterion (ABIC).

However, regularity conditions are not satisfied for deep learning architectures \citep{wei2022deep,watanabe2009algebraic}.  A false asymptotic estimate of the log marginal likelihood, and more generally of the posterior, given a prior on models, can lead to erroneous decisions in model selection~\citep{drton2017bayesian}, and so corrections to the BIC have been proposed. The widely applicable Bayesian information criterion (WBIC) relies on the correct asymptotics of the log marginal likelihood,
\begin{align*}\label{intro:WBIC}
- \log P(x_1, \dots, x_N \mid M) \;\approx\;- \log P(x_1, \dots, x_N \mid \hat\theta, M) + \lambda \log N, \tag{WBIC}
\end{align*} 
where \(\lambda\) is called the \emph{learning coefficient}. The WBIC is asymptotically consistent in selecting the most parsimonious model \citep{drton2017bayesian}. Note that Watanabe's Singular Learning Theory (SLT)~\citep{watanabe2009algebraic} works within the assumption that the data is IID, therefore all of our results are in the IID setting. Extending the results of Watanabe beyond the IID setting is an open research direction.

\paragraph{From Regular to Singular Models}
To understand when the BIC is sufficient and when WBIC is needed, we formalize the notion of regularity. The asymptotic expansion of the log marginal likelihood requires it to be re-expressed as follows,

\[
P(x_1, \dots, x_N \mid M) = \int e^{N\cdot \frac{1}{N}\sum_{i}\ln P(x_i\mid\theta, M)}\,\pi_M(\theta)\,\mathrm{d}\theta.
\]

When the sample size $N$ tends to infinity, we get 
\begin{align*}
\frac{1}{N}\sum\nolimits_{i=1}^N &\ln P(x_i\mid \theta, M)
\xrightarrow[N\to\infty]{}
\mathbb{E}_{X\sim P}\bigl[\ln P(X\mid \theta, M)\bigr]
\quad P\text{ a.s.},
\end{align*}
where $P$ is the distribution from which we draw the iid samples $x_1,\dots,x_N$ 
and $X\sim P$. Up to a term that does not depend on $\theta$, the previous limit is the negative of the Kullback Leibler divergence,
$$D_{KL}(P\Vert P_\theta)= -\mathbb{E}_{X\sim P}\bigl[\ln P(X\mid \theta, M)\bigr]+ \mathbb{E}_{X\sim P}[\ln P(X)]$$

If there exists $\rho\in\Theta$ such that $P(X)=P(X\mid \rho, M)$, then the function $g\colon \theta\mapsto D_{\mathrm{KL}}(P_{\rho}\,\|\,P_{\theta})$
is minimized at $\theta^*=\rho$ with $g(\theta^*)=0$;
we say that such a model is \textit{realizable}. If the map $\Theta_M\ni \theta \mapsto \ P(X_1,\ldots,X_N \mid \theta, M)$ is injective, then we say that the model is \textit{identifiable}. Moreover, if the Hessian of $g$ at $\theta$ is positive definite at every $\theta\in\Theta_M$, we say that the model itself is \textit{positive definite}. To make it simple, if a model is injective and positive definite, it is said to be \textit{regular}. In the regular case, BIC provides the correct asymptotic expansion of the log marginal likelihood under standard regularity conditions. If a model is not regular, then it is called \textit{singular}. In the singular case, WBIC is the correct expression. In this setting the learning coefficient is the RLCT, as we will explain just after. We now introduce the RLCT and review state-of-the-art methods and assumptions for its computation, which stem from the study of singularities of analytic functions.

\paragraph{Real Log Canonical Threshold}\label{par:RLCT} If the log-likelihood is a real analytic function, the learning coefficient coincides with the \emph{real log canonical threshold} (RLCT) of the Kullback–Leibler divergence at the realizable model, i.e., the true model that generated the samples. The RLCT is a purely algebraic quantity that is well-studied in singularity theory. 

In statistical learning, the RLCT of the Kullback-Leibler divergence of a model-truth-prior triplet is the learning coefficient, and its local version is the local learning coefficient (LLC). The global learning coefficient determines the asymptotic behavior of the log marginal likelihood, while the local learning coefficient characterizes the learning dynamics near specific parameter configurations \citep{lau2025the}. We now give a formal definition. 

\begin{definition}[Real Log Canonical Threshold] \label{rlctalg}
    Let $f:\RR^n \to \RR$ be a real analytic function defined on an open set $O \subset \RR^n$. Let $C$ be a compact subset of $O$. Then, for each $x \in C$ such that $f(x)=0$ there exists, by Theorem \ref{thm:ros},  $W \subseteq \RR^n$, an open set containing $\mathbf{0}$, an $n$-dimensional real analytic manifold  $U$, and a real analytic map $\rho : U \to W$, such that:

    $$f(\rho(u))-x)= \pm u_1^{k_1} \cdots u_n^{k_n},
    \qquad \text{ and } \qquad
    \text{Jac}(\rho) = b(u) u_1^{h_1}\cdots u_n^{h_n}.$$

    The local real log canonical threshold at $x$ is given by:
        $$\rlct_{x}(f) = \min_{1 \leq i \leq n} \frac{h_i +1}{k_i}.$$

The global RLCT is the minimum local RLCT over all points of $C$, that is to say, $$\text{RLCT}(f) \coloneq \min_{x\in C} \text{RLCT}_x(f).$$
\end{definition}

Note that the RLCT is a rational number. The RLCT is defined using a real analytic birational map, called the resolution of singularities, which, roughly speaking, represents $f$ locally as a normal crossing function. The existence of such a map for real analytic functions was famously proved by \cite{hironaka1964resolution}. We recall this theorem in Appendix~\ref{app:resolution}. Although the literature on RLCTs remains sparse, as opposed to that of its complex counterpart, the log canonical threshold (LCT) \citep{mustata2012impanga}, its algebraic properties have been studied in \cite[Chapter 4]{lin2011algebraic} and \cite{saito2007real}, and were used to classify real hyperplane singularities in \cite{kosta2024classification}. While Hironaka's theorem guarantees the existence of a resolution of singularities, computing this resolution explicitly is notoriously difficult and generally impractical \citep{bierstone2011effective}. Our work is motivated by this computational bottleneck.

\subsection{Prior Work}

The computation of learning coefficients (RLCTs of the Kullback-Leibler divergence for a given sample, model, and generating distribution) remains an active area of research. We review three main categories of work related to this problem: model-specific theoretical results, sampling-based estimation, and algebraic methods for computing RLCTs, with a particular focus on the two-dimensional case where models are parameterized by two real numbers $\theta \in \mathbb{R}^2$. 

Learning coefficients have been computed exactly for a range of statistical models through dedicated theoretical analysis. These include mixture models \citep{yamazaki2003singularities}, three-layered neural perceptrons \citep{aoyagi2005resolution}, restricted Boltzmann machines \citep{aoyagi2013learning}, as well as Bayesian networks \citep{rusakov2005asymptotic}. More recently, learning coefficients for deep linear networks of arbitrary depth \citep{aoyagi2024consideration} and factor analysis models \citep{drton2025singular} have also been calculated.  \cite{lin2011algebraic} used a Newton polygon based method to compute local learning coefficients, when the Newton polygon satisfies strong geometric constraints. Our work removes these restrictions for the two-dimensional case, providing a general algorithm for arbitrary 2D polynomials.

When exact theoretical results are unavailable, the local learning coefficient can be estimated via sampling methods. \cite{lau2025the} introduced a scalable estimator $\hat{\lambda}^{\text{SGLD}}(\theta^*)$ for the LLC, based on stochastic gradient Langevin dynamics (SGLD) \citep{welling2011bayesian} sampling of the posterior distribution and on Watanabe's widely applicable Bayesian Information Criterion (WBIC) \citep{watanabe2013widely}. While SGLD-based estimation has shown empirical success, it provides no theoretical guarantees for sampler convergence \citep{hitchcock2025global}. In the absence of known values of the learning coefficient, measuring the convergence of SGLD chains can be difficult and calibrating their hyperparameters can be costly \citep{vehtari2021rank}.

For analytic functions in two variables, a line of work beginning with \cite{varchenko1976newton} and continuing through \cite{phong1999growth} and \cite{collins2018log} established that the local RLCT of a polynomial $f(x,y)$ can be computed using the geometry of the Newton polygons of certain transformations $\tilde{f}(x,y)$ of $f(x,y)$ (see Appendix~\ref{app:newton_polygon}). \cite[Remark 3.11]{paemurru2024reading} proposes an algorithmic method, based on \cite{boehm2020classification}, but it does not terminate for certain polynomial classes (see Section) and lacks complexity bounds where it does. We resolve these issues by adapting the proofs of \cite{phong1999growth,collins2018log} to show that finitely many transformations of $f(x,y)$ suffice to compute the local RLCT of any bivariate polynomial. Our Algorithm~\ref{alg:RLCT_main} computes the exact local RLCT and terminates for any $f(x,y)$, with an explicit upper bound on the number of steps.

\section{Theoretical results}\label{sec:main_results}

In this section, we present our main theoretical results. 
We present an algorithm (Algorithm~\ref{alg:RLCT_main}) to compute the local RLCT of a bivariate polynomial at the origin and we also give a quadratic bound (Theorem~\ref{thm:complexity_main}) on its algorithmic complexity, 
depending on its degree. Furthermore, we give a bound on the arithmetic and bit complexity of Algorithm~\ref{alg:RLCT_main}.
We sketch the proofs of the various results that support the correctness and the complexity bound of the algorithm. The full proofs appear in the Appendices~\ref{sec:proof_of_main} and~\ref{sec:correctness}. 
The algorithm combines geometric and algebraic techniques and is based
on properties of the Newton polygon of a bivariate polynomial (Appendix~\ref{app:newton_polygon}).

\subsection{Preliminaries: Newton Polygons and RLCTs}\label{subsec:newton_polygons_rlcts}

Computing the resolution of singularities of a given analytic function or polynomial is notoriously expensive \citep{bierstone2011effective}. Thus it is desirable to find alternative ways to compute the RLCT. In \cite{varchenko1976newton}, Varchenko introduced a method to compute the local RLCT of an analytic function $f:\RR^2\to\RR$ at an isolated singularity, based on the geometry of its Newton polygon $\newton(f)$. This method was later generalized to arbitrary singularities \cite[Theorem 5]{phong1999growth}. A modern treatment is given in \cite{collins2018log}. The core idea of the above results is to construct an RLCT-preserving automorphism $\Phi:\KK\{x,y\}\to\KK\{x,y\}$, such that the RLCT of an analytic function $f(x,y)$ can be read on the Newton polygon $\newton(\Phi(f))$. The main theoretical contribution of our work is to make Varchenko's method effective. Indeed, in some cases, Varchenko's method may not terminate. 

We circumvent this issue by identifying precisely the cases where Varchenko's method fails to terminate, and handle them separately. We first define the Newton polygon of a function $f$ and its corresponding Newton distance $\delta_f$. 

\begin{definition}[Newton Polygon \citep{casas2000singularities}]
Let $\KK$ be a field of characteristic zero. Let $f\in \KK[[x,y]]$ be a formal power series in two variables:
\[f(x,y)= \sum\nolimits_{\alpha,\beta=0}^\infty c_{\alpha\beta} x^\alpha y^\beta \in \KK[x,y].\]
Let $D(f)=\{(\alpha,\beta)\mid c_{\alpha\beta}\neq0\} \subseteq \NN^2$ be the support of $f$. The Newton polygon, $\newton(f)$, is obtained by attaching a copy of the positive quadrant $\RR_{\geq 0}^2$ to each point of $D(f)$ and considering the convex hull of the union; that is  
	$\newton(f) \coloneqq \mathop{Conv}( D(f)\oplus \RR_{\geq0}^2)$,
	where $\oplus$ denotes the Minkowski sum.
\end{definition}

The Newton distance is a positive rational number associated to the Newton polygon $\newton(f)$ of $f(x,y)$, which measures how far the Newton polygon is from the origin.

\begin{definition}[Main face and Newton distance]
    The face $\Delta$ at which the diagonal $\mathcal{D}\{(\alpha,\beta) \in \RR^2_{\geq 0}\mid \alpha=\beta \}$ intersects the boundary $\partial \newton(f)$ of the Newton polygon is called the \textit{main face} of $f(x,y)$.
    Let $\bm{p} = (\delta,\delta) \in \QQ_{\geq0}^2$ be the point of intersection of $\mathcal{D}$ with $\partial\newton(f)$. Then, $\delta_f \coloneq \delta \in \QQ_{\geq0}$ is the \textit{Newton distance} of $f$.
\end{definition}

 We remind the reader of the notion of right equivalence of power series.

\begin{definition}[Right equivalence of power series]
Let $f,g \in \KK[[x,y]]$ be formal power series. We say that $f$ and $g$ are right-equivalent if there exists an isomorphism $\Phi : \KK[[x,y]] \to \KK[[x,y]]$, such that $f \circ \Phi = g$. 
\end{definition}

Next, we state the normalization condition on the Newton polygon $\newton(f)$ of $f(x,y)$. If a polynomial or analytic function $f(x,y)$ satisfies the normalization condition, then $\rlct_0(f)=\frac{1}{\delta_f}$.

\begin{definition}[Normalization condition \cite{phong1999growth,collins2018log}] \label{def:normalization_condition}
    Let $f(x,y) \in \RR\{x,y\}$, and let $\newton(f)$ be its Newton polygon. Let $\Delta_i$ for $i\in[K]$ be the facets of $\newton(f)$. 
    Let $F_{\Delta_i}(z)$ be its facet polynomials and let $a_{ij}$ and $m_{ij}$ for $j\in[r_i]$ be its distinct roots, and their corresponding multiplicities, respectively.
    We say that $\Delta_i$ is normalized if $\delta_f \geq m_{ij}$ for all $j\in[r_i]$.
    We say that $f(x,y)$ is normalized if the facet $\Delta_i$ of its Newton polygon $\newton(f)$ is normalized for all $i\in[K]$.
\end{definition}

\begin{proposition}[\cite{collins2018log}] \label{prop:normalized}
    Let $f\in\KK\{x,y\}$. If $f(x,y)$ is right equivalent to a normalized power series $g(x,y)$, then $\rlct_0(f)=\frac{1}{\delta_g}$.
\end{proposition}

The following proposition plays an important role in our work. The version that we present 
appears in \cite[Theorem~5]{phong1999growth} and \cite[Theorem~1.4]{collins2018log};
a first version for real isolated plane curve singularities appears in \cite[Theorem~0.6]{varchenko1976newton}.
It guarantees that for any bivariate analytic function $f$, one can construct a right equivalence $\Phi$, such that $f$ is normalized.
For completeness, we present a self-contained proof in the Appendix~\ref{subsec:proof_prop:main},
where we also show that $P_f(x)$ or $Q_f(y)$ are truncations of Puiseux roots of $f(x,y)$.

\begin{proposition}[\cite{phong1999growth,collins2018log}] \label{prop:main}
    Any analytic function $f \in \KK\{x,y\}$ is right equivalent to a normalized power series. Moreover, this right equivalence is given by a change of variable $\Phi: \KK[x,y] \to \KK[x,y], (x,y) \mapsto (x,y-P_f(x))$ or $\Phi: \KK[x,y] \to \KK[x,y], (x,y) \mapsto (x-Q_f(y),y)$, where $P_f(x)$ and $Q_f(y)$ are power series of strictly positive order in $\bar{\KK}[[x]]$ and $\bar{\KK}[[y]]$ respectively.
\end{proposition}

Proposition~\ref{prop:main} suggests that the local RLCT of $f(x,y)$ is computable up to arbitrary precision, but it is not effective, as $P_f(x)$ can have infinitely many terms. Our work resolves this issue.

\subsection{Main Results} \label{subsec:main_results}
In this work, we show that it is sufficient to compute $P_f(x)$ up to a \textit{finite} degree in order to compute the local RLCT of a polynomial $f(x,y)$ at $0$. We obtain a computable bound on this degree, and thus on the maximum number of iterations necessary for computing the local RLCT. We then provide an exact, effective algorithm, which takes as input any polynomial $f(x,y)\in \QQ[x,y]$ and outputs $\rlct_0(f) \in \QQ$. Without loss of generality, we henceforth focus only on the situation where $P_f(x)$ gives the required right equivalence. Otherwise, one can consider the change of variable $x\mapsto y$, $y \mapsto x$. We show that if $P_f(x) \in \KK[[x]]$ is such that $f(x,y-P_f(x))$ is normalized, then $P_f(x)$ is unique. Furthermore, if the coefficients of $f(x,y)$ are rational, then so are the coefficients of $P_f(x)$.
The proof of the following lemma in Appendix~\ref{subsec:proof_lem:unique}.

\begin{proposition} \label{prop:unique}
    For any unnormalized analytic function $f(x,y)\in\KK\{x,y\}$, the power series $P_f(x)$ computed in the proof of Prop.~\ref{prop:main} is unique. Moreover, if $\KK=\QQ$, then $P_f(x)\in\QQ[[x]]$.
\end{proposition}

We can thus speak of the normalizing power series of a bivariate polynomial. 

\begin{definition}[Normalizing power series]
    Let $f(x,y)$ be an analytic function. 
    We call a power series $P_f(x)$ the \textit{normalizing power series} of $f(x,y)$
    if $\tilde{f}(x,y) = f(x,y-P_f(x))$ is normalized.
\end{definition}

\begin{remark}\label{rem:two_cases}
    Recall that a $y$-root of $f(x,y)$ is an element $\phi(x)$ of the algebraic closure of $\KK[x]$, such that $f(x,\phi(x))=0$ (see Appendix~\ref{app:yroots}). 
    In the course of the proof of Proposition~\ref{prop:main}, it is shown that $P_f(x)$ corresponds either:
    \begin{enumerate}[label=(\arabic*)]
        \item to a finite truncation of some  Puiseux root of strictly positive order of $f(x,y)$, or,
        \item to a power series of strictly positive order $\phi(x) \in \bar\KK[[x]]\setminus\KK[x]$ which is a $y$-root of $f(x,y)$.
    \end{enumerate}
     Using the Newton-Puiseux algorithm, $P_f(x)$ can thus be computed entirely, in the first case, or up to an arbitrary number of terms in the second case.
\end{remark}

\begin{comment}

As we can see, the only obstacle to an effective algorithm is Case $2$ of Remark~\ref{rem:two_cases}, that is when the normalizing power series $P_f(x)$ of $f(x,y)$ has infinitely many terms. As a simple example, consider $f(x,y)=y^2+y+x \in \QQ[x]$. It is known (for example, by \cite[Theorem 1]{banderier2013coefficients}), that $f(x,y) = y^2 + y + x$ has a unique root of strictly positive order $\phi(x) \in \QQ[[x]]\setminus\QQ[x]$. As a root of $f(x,y)$, $\phi(x)$ has multiplicity $1$, and therefore $m^{(i)}=1$ for all $i \in \NN$. However, it can be shown that $\delta_{f^{(i)}} < m^{(i)}=1$, where $f^{(i)}$ is a transformation of $f$, as defined in the sketch of proof of Proposition~\ref{prop:main}, and so $f^{(i)}$ would only be normalized after an infinite amount of steps.

\end{comment}

We first define the finite part $P^{fin}_f(x)$ of $P_f(x)$, where we use the notion of singular parts of Puiseux series (Definition~\ref{def:singpart}). For the remainder of the section, we let $\KK =\QQ$.

\begin{definition}[Finite part of a normalizing power series]\label{def:finitepart}
    Let $f(x,y)\in\QQ[x,y]$, with normalizing power series $P_f(x)$. If there exists a $y$-root $\phi(x)$ of $f(x,y)$, such that $P_f(x)=\phi(x)$, then we define the finite part $\finitepart{f}{x}$ of $P_f(x)$ to be the singular part (Definition~\ref{def:singpart}) of $P_f(x)$; that is  
    	$\finitepart{f}{x}\coloneqq S_{P_f}(x)$.
    Otherwise, we let $\finitepart{f}{x}\coloneqq P_f(x)$.
\end{definition}

The following result justifies the algorithm we propose in the next section. It establishes that a finite part of $P_f(x)$, which distinguishes it from all the other Puiseux roots $\phi(x)$ of strictly positive order, is sufficient to compute the local RLCT of a polynomial $f(x,y)$.

\begin{theorem}\label{thm:finite_part_main}
    Let $f(x,y)\in\QQ[x,y]$, with normalizing power series $P_f(x) \in \QQ[[x]]$. Let $\finitepart{f}{x}$ be its finite part. Let $\tilde{f}(x,y)= f(x,y-\finitepart{f}{x})$. Let $\Delta$ be the main face of $\newton(f)$. If $\tilde{f}$ is normalized, then $\rlct_0(f)=\frac{1}{\delta_{\tilde{f}}}$. If not, $\rlct_0(f)=\frac{1}{m}$, where $m>\delta_f$ is the largest multiplicity of the roots of $\tilde{F}_\Delta(z)$.
\end{theorem}

\begin{proof}
This theorem follows from the unicity of $P_f(x)$, and from the maximality of the degree of singular parts of Puiseux roots of polynomials. We defer the detailed proof to Appendix~\ref{subsec:proof_thm:finite_part_main}.
\end{proof}

This result tells us precisely under which conditions Varchenko's method fails to terminate, and recovers the RLCT even in those pathological cases.
Our second theoretical result generalizes a univariate root separation bound \citep{tsigaridas2008complexity}, to an upper bound on the degree of $\finitepart{f}{x}$. This allows us to define a stopping criterion for our algorithm and thus is also a bound on the number of steps that our algorithm takes before terminating. 

\begin{theorem}\label{thm:complexity_main}
    Let $f(x,y)\in\QQ[x,y]$, such that $\deg_yf=d_y$ and $\deg_xf=d_x$. Consider the finite part of its normalizing power series, $\finitepart{f}{x}$. Let $d$ be the degree of $\finitepart{f}{x}$. Then $d < (d_y+\frac{1}{2})d_x$.
\end{theorem}

\begin{sproof}
    The proof in full detail can be found in Appendix~\ref{subsec:proof_thm:complexity_main}. Suppose that the finite part $\finitepart{f}{x}$ is shared by at least two Puiseux roots of $f(x,y)$, say $\phi(x)$ and $\psi(x)$. The order of the difference between $\phi$ and $\psi$, $\ord(\phi(x)-\psi(x))$ is an upper bound on the degree of $\finitepart{f}{x}$. We are thus interested in finding upper bounds on the order of the differences of the roots of $f(x,y)$, based on the degree in $x$ and $y$ of $f(x,y)$. This is analogous to the root separation problem in the univariate case: finding a lower bound on the absolute difference between the roots of a given polynomial $f(z)$. We generalize the approach of \cite{tsigaridas2008complexity}, which lower bounds the root separation of a polynomial $g(z) \in \ZZ[z]$, using the discriminant of $g$. We argue that the degree in $x$ of the discriminant of $f(x,y)$ gives a similar upper bound in the bivariate case. 
\end{sproof}

Theorem~\ref{thm:complexity_main} provides a bound on the number of steps necessary to implement the check in Theorem~\ref{thm:finite_part_main}. Combining Theorem~\ref{thm:finite_part_main} and Theorem~\ref{thm:complexity_main}, we are now ready to construct our algorithm.

\section{An exact algorithm to compute the local RLCT of 2D polynomials}
\begin{algorithm}[h!]
\caption{\texttt{ComputeRLCT2D}}
\label{alg:RLCT_main}
\LinesNumbered
\SetAlgoLined
\KwIn{A polynomial $f(x,y)\in \QQ[x,y]$.}
\KwOut{$\rlct_0(f)$}

$d_y \gets \deg_yf$\;

$d_x \gets \deg_x f$\;

$B\gets (d_y+\frac{1}{2})d_x$\;

$d \gets 0$\;

\While{$B > 0$}{
    Compute the Newton polygon $\newton(f)$ of $f$\;
    
    Compute the Newton distance $\delta_f$ of $f$\;
    
    \uIf{$\texttt{Normalized}(f,\newton(f),\delta_f)=\text{True}$}{
        \Return{$\frac{1}{\delta_f}$}\;
        }
    \ElseIf{$\texttt{Normalized}(f,\newton(f),\delta_f)=(b,p,m)$}{
            $f(x,y)\gets f(x,y+bx^p)$\;
            
            $B\gets B-(p-d)$\;

            $d \gets p$\;
            }
    }
\uIf{$\texttt{Normalized}(f,\newton(f),\delta_f)=\text{True}$}{
    \Return{$\frac{1}{\delta_f}$}\;
    }
\ElseIf{$\texttt{Normalized}(f,\newton(f),\delta_f)=(b,p,m)$}{
        \Return{$\frac{1}{m}$}\;
        }
\end{algorithm}

In this section, we propose an algorithm for computing the local RLCT at the origin of any polynomial $f\in\QQ[x,y]$ (Algorithm~\ref{alg:RLCT_main}), and we discuss it's arithmetic and bit complexity. This algorithm implements the constructive proof of the existence of a normalizing power series, as described in \cite{phong1999growth,collins2018log}, but includes a stopping criterion derived from Theorem~\ref{thm:complexity_main}, such that the normalizing power series is only computed up to $\finitepart{f}{x}$. By Theorem~\ref{thm:finite_part_main}, this is sufficient to compute $\rlct_0(f)$. The definition is self-contained, except for the \texttt{Normalized} subroutine, which checks if a polynomial $f$ satisfies the normalization condition (Definition \ref{def:normalization_condition}). The definition of the \texttt{Normalized} subroutine is deferred to Appendix~\ref{app:aux_alg}. We prove the correctness of Algorithm~\ref{alg:RLCT_main}. The proof is deferred to Appendix~\ref{subsec:proof_thm:correctness}.
	
\begin{theorem}[Correctness]\label{thm:correctness}
    Let $f(x,y) \in \QQ[x,y]$. Then Algorithm~\ref{alg:RLCT_main} computes the real log canonical threshold at the origin of $f(x,y)$.
\end{theorem}

We obtain upper bounds on the number of iterations required to compute the RLCT using Algorithm~\ref{alg:RLCT_main}. 

\begin{corollary}
\label{cor:rlctiters}
    Let $f(x,y)\in\QQ[x,y]$, such that $\deg_yf=d_y$ and $\deg_xf=d_x$. Then, Algorithm~\ref{alg:RLCT_main} computes the $\rlct_0(f)$ in less than $(d_y+\tfrac{1}{2})d_x$ iterations.
\end{corollary}

\begin{proof}
    Theorem~\ref{thm:correctness} asserts that Algorithm~\ref{alg:RLCT_main} computes the $\rlct_0(f)$. The number of iterations of the while loop (lines \texttt{14-15}) is upper bounded by $B =(d_y+\tfrac{1}{2})d_x$.
\end{proof}

Furthermore, we derive an upper bound on the arithmetic and bit complexity of Algorithm~\ref{alg:RLCT_main}. Those bounds are obtained by adapting results from the literature on the complexity of computing singular parts of Puiseux series, e.g. \cite{walsh2000polynomial,poteaux2021computing}. While Algorithm~\ref{alg:RLCT_main} is an algorithm that computes certain terms of the singular part of a Puiseux series, it is specific in the sense that only rational terms enter in the computation (Proposition~\ref{prop:unique}).As algebraic numbers are more costly to represent, the bounds previously obtained in the literature are greatly improved in our particular setting. Nevertheless, our algorithm is not optimal as an algorithm which compute Puiseux series, and several state of the art techniques from e.g. \cite{poteaux2021computing} could be employed to bring down the computational cost even further. We now state our theorems on the arithmetic complexity of Algorithm~\ref{alg:RLCT_main}. Full proofs are found in Appendix~\ref{sec:correctness}.
\begin{theorem}[Arithmetic Complexity of Algorithm~\ref{alg:RLCT_main}]\label{thm:arithmetic_complexity}
Algorithm~\ref{alg:RLCT_main} requires $\tilde{O}(d_x^2d_y^4)$ arithmetic operations in $\QQ$.
\end{theorem}
Lastly, we state our theorem on the bit complexity of Algorithm~\ref{alg:RLCT_main}.
\begin{theorem}[Bit Complexity of Algorithm~\ref{alg:RLCT_main}]\label{thm:bit_complexity}
Assuming that the input polynomial is primitive with integer coefficients, the bit complexity of Algorithm~\ref{alg:RLCT_main} is polynomial up to logarithmic factors in $d_x$, $d_y$ and $\tau$ where $\tau$ is the maximal bit size of the coefficients of the input polynomial. More precisely, the bit complexity of Algorithm~\ref{alg:RLCT_main} is $\tilde{O}(d_y^8d_x^3(\tau +d_x+d_y))$.
\end{theorem}

{\color{black}
\section{Application to polynomial models}\label{sec:polynomial_param}

We demonstrate the applicability of our algorithm by computing local learning coefficients of biparametric models. We define an equivalence relation on the space of real functions which allows us to apply our algorithm, even when the Kullback-Leibler distance of the considered model is not polynomial. We illustrate this for Polynomial Neural Networks in Section~\ref{subsec:PNNs}.

\subsection{Polynomial Neural Networks}\label{subsec:PNNs}

Polynomial neural networks (PNNs) are a class of neural network models whose activations are %piecewise
monomial functions, which recently have shown state-of-the-art performance on a variety of tasks, from image generation and classification \citep{chrysos2020p,huang2003face}, to trading signals forecasting \citep{ghazali2011dynamic}, signal representation \citep{yang2022polynomial} and to the resolution of inverse problems in physics \citep{bu2021quadratic}, among others.

\begin{definition}[Polynomial Neural Network \citep{kubjas2024geometry}]

A (bias-less) \emph{polynomial neural network} $f_{\bm{\theta}}$, with architecture $d=(d_0,\ldots,d_L)$ is a function $f_{\bm{\theta}}:\RR^{d_0} \to \RR^{d_L}$, defined as:
\[
f_{\bm{\theta}} = W_L \circ\sigma_{L-1}\circ W_{L-1} \circ\sigma_{L-2}\circ \cdots\circ\sigma_{1} \circ W_1,
\]
where $W_i \in \RR^{d_i \times d_{i-1}}$ are linear maps, and the monomial activation functions $\sigma_i : \RR^n \to \RR^n$ act component-wise and are given by:
\[\sigma_i(x) = (x_1^r,\ldots,x_n^r).\] The integer $r$ is called the activation degree of the network. The parameters $\bm{\theta}$ are given by the entries of the matrices (or weights) $W_1,\ldots,W_L$. 
\end{definition}

As PNNs define an algebraic map from weights to polynomials, their polynomial nature makes them particularly amenable to the tools of algebraic geometry. For example, authors have used algebraic dimensions \citep{kileel2019expressive}, algebraic degrees \citep{kubjas2024geometry}, singularity theory \citep{shahverdi2025learning}, and low-rank tensor decompositions \citep{usevich2025identifiability}, to study the expressivity of PNNs, to bound their number of learnable functions, characterize subnetworks and analyze their identifiability, respectively. In this section, we apply Algorithm \ref{alg:RLCT_main} to a regression model induced from a PNN with repeated weights, and compute its local learning coefficient at the origin.

\subsection{Experimental Setup}

We consider PNNs of varying depth $L$, with parameters $\theta=(\theta_1,\theta_2) \in \Theta \subset \RR^2$ and repeat $L$ times the weight 
\[
W  = \begin{pmatrix}
    p_{11}(\theta_1,\theta_2) & p_{12}(\theta_1,\theta_2) \\
    p_{12}(\theta_1,\theta_2) & p_{22}(\theta_1,\theta_2) 
\end{pmatrix}\in \QQ[\theta_1,\theta_2]^{2\times2}.
\]
We fix the activation degree of the network to be $r=2$, such that:

\begin{equation}
\label{eq:PNN_repeated_weights}
f_{\theta}(x_1,x_2)=  W \circ\sigma_{L-1}\circ W \circ\sigma_{L-2}\circ \cdots\circ\sigma_{1} \circ W(x_1,x_2).
\end{equation}

Write $x=(x_1,x_2)$. We will furthermore require that, for any polynomial entry of the weight $p_{ij}(\theta_1,\theta_2)$, that $p_{ij}(0,0)=0$. Letting $\theta^* = (0,0)$, we observe that $f_{\theta^*}(x)=(0,0)$.

As in \citep{wei2022deep,aoyagi2024consideration}, we focus on the regression task. We consider the family of models:

\[\forall y,x\in \mathbb{R}^2 \quad
p(y\mid x,\theta) = \frac{1}{2\pi} \exp(-\frac{1}{2}\lVert y - f_{\bm{\theta}}(x)\rVert^2).
\]

Let the parameter space $\Theta$ be a compact subset of $\RR^2$, containing the origin. Let $q(y \mid x)$ be the true generating distribution of the data. Assume that a continuous distribution $q(x)$ on $\RR^2$ is fixed, with respect to which $p(x, y) = p(y|x)q(x)$ and $q(x, y) = q(y|x)q(x)$. We assume that $p(y \mid x,\theta)$ is a realizable model, that is to say, the set $\Theta_0 = \{ \theta \in \Theta \mid p(y\mid x, \theta)=q(y\mid x)\}$ is non-empty. Furthermore, let us assume that $(0,0)=\theta^*\in \Theta_0$. We let $K(\theta)=\mathbb{E}_{q(x)}\left[D_{KL}(p(y \mid x, \theta) \,\lVert \,q(y\mid x))\right]$.  Then, by straightforward calculations \citep[Appendix A.1]{wei2022deep},

\begin{equation*}
\begin{split}
 K(\theta) &= \int_X \lVert f_\theta(x) - f_{\theta^*}(x)\rVert^2q(x)dx 
  = \int_X \lVert f_\theta(x) \rVert^2q(x)dx.
\end{split}
\end{equation*}

Recall that the local learning coefficient at $\theta_0$ of $p(y \mid x, \theta_0)$ is none other than the local real log canonical threshold of $K(\theta)$ at $\theta_0$. However, $K(\theta)$ is not a polynomial, and we cannot apply Algorithm~\ref{alg:RLCT_main}. This is resolved by considering an auxiliary "contact-equivalent" polynomial $H(\theta)$ whose local RLCT at the origin is equal to that of $K(\theta)$ (Proposition~\ref{prop:rlct_contact_invariant}). Appendix~\ref{appendix:contact_equivalence} discusses contact equivalence and how to compute contact equivalent polynomials. Proposition~\ref{prop:aoyagi} is readily applicable to $K(\theta)$. The contact-equivalent polynomial $H(\theta)$ is obtained explicitely from the coefficents of $f_\theta(x)$, and is an exact polynomial, rather than an empirical approximation. Moreover, because the PNN map has rational polynomial entries, the resulting contact-equivalent polynomial has rational coefficients and is a valid input to Algorithm~\ref{alg:RLCT_main}. For general neural networks, the parameter-to-model map is not a polynomial mapping, and novel techniques maybe be required in order to compute contact-equivalent polynomials, which are not solved by the present work.

\subsection{Examples} \label{subsec:examples}

Let 
$W = 
\begin{psmallmatrix}
    \theta_1 + \theta_2 & \,\theta_1^2 \\
    \theta_1^2 & \,\theta_2^2 
\end{psmallmatrix}$. For each $1 \leq L \leq  5$ and $2 \leq  r \leq 4$, we compute $H_{L,r}(\theta)$, the sum-of-squares polynomial contact equivalent to $K_{L,r}(\theta)$, the Kullback--Leibler distance of the depth-$L$, activation-$r$ PNN $f_\theta$ of Equation~\ref{eq:PNN_repeated_weights}. The $H_{L,r}(\theta)$ are not normalized, so Algorithm~\ref{alg:RLCT_main} takes at least one step before outputting $\lambda_{L,r} = \rlct_0(H_{L,r})$; such degenerate polynomials lie outside the reach of prior Newton polygon methods like \citep[Section~4.2.1]{lin2011algebraic}.

\begin{table}[H]
\centering
\caption{Local RLCT at the origin of $H_{L,r}(\theta)$: exact rationals $\lambda_{L,r}$ from Algorithm~\ref{alg:RLCT_main} (\texttt{SageMath}, Intel Core Ultra~7) versus SGLD estimates $\hat\lambda^{\text{SGLD}}$ (relative error $|\hat\lambda^{\text{SGLD}}-\lambda_{L,r}|/\lambda_{L,r}$ in parentheses; $\gamma\!=\!1$, $\epsilon\!=\!10^{-5}$, $C\!=\!5$ chains, $T\!=\!10{,}000$; Google Colab, 2~vCPUs). \texttt{Exp} is the time to expand $H_{L,r}(\theta)$; \texttt{Alg.~1} is Algorithm~\ref{alg:RLCT_main} alone. \texttt{NaN} marks numerical instability; $>$1\,hr marks runs that did not finish for Algorithm~\ref{alg:RLCT_main}.}
\label{tab:rlcts}
\scriptsize
\begin{tabular}{ |c|l|c|c|c|c|c| }
\hline
& & $L=1$  & $L=2$ & $L=3$ & $L=4$  & $L=5$ \\
\hline
\multirow{6}{*}{$r=2$}
& $\lambda_{L,r}$           & $3/4$   & $1/4$   & $3/28$  & $1/20$  & $3/124$ \\
& \texttt{Exp}              & 0.001s  & 0.004s  & 0.031s  & 2.09s   & 153.7s  \\
& \texttt{Alg.~1}           & 0.045s  & 0.059s  & 0.271s  & 3.17s   & 66.4s   \\
& \texttt{Total}            & 0.047s  & 0.063s  & 0.302s  & 5.26s   & 220.2s  \\
\cline{2-7}
& $\hat\lambda^{\text{SGLD}}$ & .677\,{\scriptsize(9.7\%)} & .154\,{\scriptsize(38.4\%)} & .038\,{\scriptsize(64.5\%)} & .014\,{\scriptsize(72.0\%)} & .006\,{\scriptsize(75.2\%)} \\
& \texttt{Time (SGLD)}      & 461s    & 430s    & 421s    & 473s    & 581s    \\
\hline
\multirow{6}{*}{$r=3$}
& $\lambda_{L,r}$           & $3/4$   & $3/16$  & $3/52$  & $3/160$ & $>$1\,hr \\
& \texttt{Exp}              & 0.001s  & 0.007s  & 1.92s   & 868s    & --       \\
& \texttt{Alg.~1}           & 0.045s  & 0.141s  & 2.73s   & 304s    & --       \\
& \texttt{Total}            & 0.047s  & 0.148s  & 4.65s   & 1173s   & --       \\
\cline{2-7}
& $\hat\lambda^{\text{SGLD}}$ & .677\,{\scriptsize(9.7\%)} & .081\,{\scriptsize(56.8\%)} & .025\,{\scriptsize(56.7\%)} & .007\,{\scriptsize(62.7\%)} & \texttt{NaN} \\
& \texttt{Time (SGLD)}      & 576s    & 474s    & 485s    & 495s    & --       \\
\hline
\multirow{6}{*}{$r=4$}
& $\lambda_{L,r}$           & $3/4$   & $3/20$  & $1/28$  & $>$1\,hr & $>$1\,hr \\
& \texttt{Exp}              & 0.001s  & 0.004s  & 6.13s   & --       & --       \\
& \texttt{Alg.~1}           & 0.045s  & 0.079s  & 12.5s   & --       & --       \\
& \texttt{Total}            & 0.047s  & 0.083s  & 18.6s   & --       & --       \\
\cline{2-7}
& $\hat\lambda^{\text{SGLD}}$ & .677\,{\scriptsize(9.7\%)} & .093\,{\scriptsize(38.0\%)} & .019\,{\scriptsize(46.8\%)} & \texttt{NaN} & \texttt{NaN} \\
& \texttt{Time (SGLD)}      & 456s    & 431s    & 443s    & --       & --       \\
\hline
\end{tabular}
\end{table}

The pre-processing step of expanding $H_{L,r}(\theta)$ becomes the computational bottleneck at greater depth and activation degree, but is independent of the RLCT computation itself, which remains competitive with SGLD sampling and requires no hyperparameter tuning. Timings are indicative only, since the two methods were run on different hardware. As the number of layers increases, so do the degree and the size of the coefficients of $H_{L,r}(\theta)$, leading to increased wall-clock time. The local learning coefficients decrease as a function of the depth $L$ and the activation degree $r$, suggesting that PNNs with repeated weights become more degenerate as the number of layers increases: this is likely because repeated weights increase the multiplicity of the origin as a zero of $K_{L,r}(\theta)$. In fact, all computed values agree with $\lambda_{L,r} = \frac{3(r-1)}{4(r^L-1)}$. The exact values also serve as benchmarks for sampling-based estimation: across all configurations, SGLD underestimates the local learning coefficient, and its relative error tends to worsen with depth, from $9.7\%$ at $L=1$ to $75.2\%$ at $L=5$ for $r=2$, with numerical instability for the largest configurations. We note that these estimates use a single fixed choice of hyperparameters, to which SGLD-based estimators are known to be sensitive. To run our SGLD experiments, we used Timaeus' \texttt{DevInterp} library~\cite{devinterp2026}.

\section{Discussion} \label{sec:discussion}

\textbf{Towards exact algorithms for the learning theory of low-dimensional models:} We developed an effective algorithm that computes exactly the real log canonical threshold of any input polynomial $f(x,y) \in \QQ[x,y]$ in a number of steps quadratic in $\deg f$, although the tightness of this bound remains open. The algorithm is model-agnostic, broadening the class of two-dimensional models for which local learning coefficients can be computed. 

We demonstrated it by computing local learning coefficients of PNNs with repeated weights of increasing depth. Exact values are crucial for model selection in smaller models and can be used to calibrate sampling-based estimators \citep{lau2025the} for larger models. Effective model complexity, as measured by the RLCT, can decrease with depth - revealing non-trivial identifiability structure already in the two-dimensional case.

\textbf{From local to global: } Algorithm~\ref{alg:RLCT_main} computes the RLCT of $f$ at the origin. In
applications, however, one needs the global RLCT of $H\in\QQ[x,y]$ over a
parameter region $W\subseteq\mathbb{R}^2$, namely
$\rlct_W(H)=\min_{p\in V_{\mathbb{R}}(H)\cap W}\rlct_p(H)$. In the bivariate
case, locality is not a fundamental obstacle. Away from finitely many points,
the local RLCT can be read off from the square-free factorization of $H$. At
the remaining points, we translate the point to the origin and apply
Algorithm~\ref{alg:RLCT_main}.

Let $H=c\prod_{k=1}^{r}H_k^{k}$ be the square-free factorization of $H$ over
$\QQ$. Here the $H_k$ are squarefree and pairwise coprime. Write
$H_{\mathrm{red}}=\prod_k H_k$ for the square-free part. Suppose that $p$ is a smooth real zero of $H_\mathrm{red}$. Then $p$ lies in exactly one component $H_k$ of $H_\mathrm{red}$, with $H_j(p)\neq0$ for all $j\neq k$. Near $p$, we can therefore write $H=u\cdot H_k^{k}$ with $u(p)\neq 0$. Taking $H_k$ as a local coordinate gives $\rlct_p(H)=1/k$. If $p$ is not smooth, then it belongs to the real singular locus $\Sigma = V_\RR(\langle H_\mathrm{red},\partial_xH_\mathrm{red},\partial_yH_\mathrm{red}\rangle)$ of $H_\mathrm{red}$. The real singular locus $\Sigma$ is zero-dimensional. Its points can be computed using a Rational Univariate Representation (RUR) \cite{rouillier1999solving}. Once we have computed a point $p\in \Sigma$, we can compute the local $\text{RLCT}$ at $p$ after translating $p$ to the origin. Finally, the global $\text{RLCT}$ of $H$ can be computed as:
\[
  \rlct_W(H)=\min\Big(\big\{\tfrac{1}{k} : V_{\mathbb{R}}(H_k)\cap W\not\subseteq\Sigma\big\}
  \;\cup\;\big\{\rlct_p(H) : p\in\Sigma\cap W\big\}\Big).
\]

One issue with this approach is that $p$ may not be a rational point and that $H$ may no longer be a rational polynomial after a translation. In that case, we can represent it using rational polynomial arithmetic over the quotient algebra induced by the RUR. Adapting our algorithm to this case remains an open direction for future work.

\textbf{Applications to two-parameter models.}
The main limitation of Algorithm~\ref{alg:RLCT_main} is that it applies only to
models with two parameters, for which a polynomial contact-equivalent form of the Kullback-Leibler Distance can be obtained. However, two-parameter
statistical models remain common in the experimental sciences and are still
actively studied. Examples include the Michaelis--Menten model of enzyme kinetics
with Gaussian noise~\citep{o2024nonlinear}, the Kimura two-parameter model of DNA
substitution in evolutionary biology~\citep{nishimaki2019extension}, and the
Gompertz model of tumor growth~\citep{vaghi2020population}. Another example is the two-parameter logistic model of
item response theory~\citep{thissen2001item}, restricted to a single test item,
which has recently been studied from the point of view of SLT~\citep{suzuki2026singularities}. Our work could apply to any of these models, provided that we can construct a contact-equivalent polynomial to their Kullback-Leibler distances.

\textbf{Low-dimensional structure in high-dimensional models.}
Our algorithm can also be applied to two-dimensional slices of higher-dimensional
models. Restricting attention to low-dimensional subspaces of the loss landscape
is a common way to study large models~\citep{li2018visualizing, li2018measuring},
and there is evidence that such slices capture much of what happens during
training. For example, \citet{li2018visualizing} report that between 40\% and 90\%
of the variation in the optimization trajectories they study lies in a
two-dimensional subspace (Section~7.2). Closer to our setting,
\citet{wang2025differentiation} introduce the \emph{weight-refined} local learning
coefficient of Transformers, defined by freezing a subset of the weights and computing the local
learning coefficient over the remaining ones. They estimate it by sampling, with
no exact values available for comparison. When all but two weights are frozen,
our method computes this quantity exactly in the contact-equivalent polynomial case, and could serve as ground truth for
such estimators.

\textbf{From Two Parameters to $N$ Parameters:} 
From the point of view of computational algebraic geometry, Algorithm~\ref{alg:RLCT_main} is, to our knowledge, the first effective algorithm for computing local RLCTs of bivariate polynomials. In arbitrary dimension~$N$, we know of no algorithm for the RLCT that avoids resolution of singularities, as the analytic change of variable procedure developed by~\cite{varchenko1976newton} does not extend readily to higher dimensions. For the complex LCT, by contrast, algorithms do exist. They go through the Bernstein--Sato polynomial $b_f(s)$, which can be computed with the algorithms of Oaku~\cite{oaku1997algorithm} and of Brian\c{c}on--Maisonobe~\cite{briancon2002remarques}; see \cite{ucha2004computation,levandovskyy2008computational,andres2010effective} for implementations and comparisons. These algorithms compute Gr\"obner bases of left ideals in extensions of the Weyl algebra, which have more than $2N$ variables. Moreover, Gr\"obner bases in Weyl algebras have doubly exponential worst-case degree bounds~\cite{aschenbrenner2009degree}. The LCT is then the negative of the largest root of $b_f(s)$~\cite{kollar1997singularities}. This approach does not carry over to the RLCT. Saito showed that $-\mathrm{RLCT}(f)$ is also a root of $b_f(s)$, but it need not be the largest one~\cite{saito2007real}, and we know of no way to tell which root it is. Although Hironaka's resolution of singularities could be used to compute the RLCT, the best known complexity bound for Hironaka's resolution algorithm only places it in the Grzegorczyk class $\mathcal{E}^{m+3}$, where $m$ is the dimension~\cite{bierstone2011effective}. Our approach uses neither Gr\"obner bases nor Hironaka's resolution of singularities. We believe that it could serve as a stepping stone toward Gr\"obner-free methods in higher dimensions, that also avoid Hironaka-style resolution, in the spirit of the more elementary approach of~\cite{collins2013multi}.

\bibliographystyle{tmlr}
\bibliography{bib}

\appendix

\section{Hironaka's Resolution of Singularities}\label{app:resolution}

Here we recall the statement of Hironaka's resolution of singularities, following \cite{watanabe2007almost}.

\begin{theorem}[Resolution of singularities, \citep{hironaka1964resolution}]
\label{thm:ros}
    Let $f:\RR^n \to \RR$ be a non-constant real analytic function, such that $f(\mathbf{0})=0$. Let $V(f) = \{ x\in\RR^n \mid f(x)=0\}$. Then there exists $W \subseteq \RR^n$ an open set containing $\mathbf{0}$, $U$ an $n$-dimensional real analytic manifold and $\rho : U \to W$ is a real analytic map, such that the triplet $(W,U,\rho)$ satisfies the following conditions:

    \begin{enumerate}
        \item $\rho$ is a proper map.
        \item Writing $W_0 = V(f) \cap W$ and $U_0 = V(f\circ \rho) \cap U$, $\rho: U \setminus U_0 \to W \setminus W_0$ is a real analytic isomorphism.
        \item For any $p\in U_0$, there exists a local coordinate $(u_1, \ldots,u_n)$ of $U$ in which $p$ is the origin and $$f(\rho(u))=S u_1^{k_1} \cdots u_n^{k_n}$$
        where $S = 1$ or $S=-1$ is a constant, $k_1, \ldots, k_n$ are non-negative integers, and the Jacobian of $\rho$ satisfies $$\text{Jac}(\rho) = b(u) u_1^{h_1}\cdots u_n^{h_n},$$ where $b(u) \neq 0$ is a real analytic function not vanishing at $p$ and $h_1,\ldots,h_n$ are non-negative integers.
    \end{enumerate}
\end{theorem}

\section{Polynomial Contact Equivalence}\label{appendix:contact_equivalence}

In this section, a more analytic definition of the local RLCT will be particularly useful. 

\begin{definition}[Analytic definition of the real log canonical threshold \citep{collins2018log}] \label{rlctan}
Let $f:\RR^n \to \RR$ be a real analytic function. The local real log canonical threshold at $x$ is defined as: 
$$ \rlct_{x}(f) = \sup \Bigl\{ s \in \RR_{\geq 0} \mid \exists \epsilon>0 \, \text{s.t.} \, \int_{B_\epsilon(x)} \lvert f(x) \rvert^{-s} dx < \infty \Bigr\}$$
\end{definition}

Recall that (Section~\ref{par:RLCT}) the local learning coefficient of a parametric model $p(x\mid\theta)$ at a parameter $\theta^*$ is the local RLCT of the Kullback-Leibler distance $K(\theta) = D_{KL}(p(x\vert\theta) \Vert q(x))$ at $\theta^*$, where $q(x)$ is the true data generating distribution. $K(\theta)$ is not necessarily a polynomial function, and we cannot directly apply Algorithm~\ref{alg:RLCT_main}. A workaround is to find a polynomial $H(\theta) \in \RR[\theta]$, equivalent to $K(\theta)$ in the following sense.

\begin{definition}[Contact equivalence {\cite[Definition 5.6.5]{ruas2020basics}}] Two real functions $f,g:\RR^n\to\RR$ are contact equivalent a point $x_0 \in \RR^n$ if there exists a neighbourhood $\mathcal{U}_{x_0}$ of $x_0$ in $\RR^n$ such that, there exists positive constants $c_1,c_2>0$, such that for all $x \in \mathcal{U}_{x_0}$, we have:
$$ c_1 \lvert g(x) \rvert \leq \lvert f(x) \rvert \leq c_2 \lvert g(x) \rvert$$
and $f(x)g(x)\geq 0$. We write $f(x) \sim g(x)$.

\end{definition}

Contact equivalence at a point $x_0$ preserves the local real log canonical threshold at $x_0$.

\begin{proposition}\label{prop:rlct_contact_invariant}
Let $f,g:\RR^n \to \RR$ s.t. $f \sim g$ at $x_0$. Then, $\rlct_{x_0} (f) = \rlct_{x_0} (g)$.
\end{proposition}

\begin{proof}
For simplicity, let $x_0=0$ be the origin in $\RR^n$ and let $f,g:\RR^n\to\RR$ be positive functions such that $f \sim g$. Let $\rlct_0(f)=\lambda$. Since $f \sim g$, there exists a positive constant $c$ such that $f(x)\leq cg(x)$ for all $x$ in some neighborhood of the origin. For all $s\in(0,\lambda)$, there exists $\epsilon>0$ sufficiently small such that $\infty > c^s\int_{B_{\epsilon}(0)} f(x)^{-s}dx\geq \int_{B_{\epsilon}(0)} g(x)^{-s}dx $ and $g(x)^{-s}$ is integrable. Therefore, $$\Bigl\{ s \in \RR_{\geq 0} \mid \exists \epsilon>0 \, \text{s.t.} \, \int_{B_\epsilon(0)}  f(x)^{-s} dx < \infty \Bigr\} \subseteq \Bigl\{ s \in \RR_{\geq 0} \mid \exists \epsilon>0 \, \text{s.t.} \, \int_{B_\epsilon(0)}  g(x)^{-s} dx < \infty \Bigr\}$$ and hence $\rlct_0(f)\leq \rlct_0(g)$. Now, using the left side inequality of the equivalence relation, we find in the same fashion that $\rlct_0(f) \geq \rlct_0(g)$, therefore they must be equal.
\end{proof}

\begin{remark}
    The real log canonical threshold is an invariant of a broader type of equivalence of functions, namely \textit{bi-Lipschitz $\mathcal{K}$ equivalence} \cite[Theorem 7.3]{bivia2016mixed}. Two function germs (i.e. defined locally around the origin) $f,g: (\RR^n,0) \to (\RR,0)$ are said to be bi-Lipschitz $\mathcal{K}$ equivalent if there exists a bi-Lipschitz homeomorphism $\varphi: (\RR^n,0) \to (\RR^n,0)$ and a bi-Lipschitz homeomorphism $\Phi: (\RR^n \times \RR,0) \to (\RR^n \times \RR,0), (x,y) \mapsto (\varphi(x),\phi(x,y))$, such that $\Phi(\RR^n \times \{0\})= \RR^n \times \{0\}$ and $\phi(x,f(x))=g(\varphi(x))$ for all $x$ in a neighbourhood of the origin. A weaker equivalence is that of bi-Lipschitz $\mathcal{C}$ equivalence, which requires that $\varphi = \text{id}$. In \cite[Theorem 5.6.7]{ruas2020basics}, it is shown that two Lipschitz functions are of the same contact if and only if they are bi-Lipschitz $\mathcal{C}$ equivalent.

\end{remark}

The following proposition will be used to construct a polynomial equivalent to the Kullback-Leibler distance of the model, when the latter can be written in a specific form.  

\begin{proposition}[\cite{aoyagi2024consideration}, Theorem $1$]\label{prop:aoyagi}
Let $x \in \RR^n$ and let $w \in \RR^m$. Using multi-index notation, write $x^{\bm{\alpha}} = x_1^{\alpha_1}\cdot\ldots\cdot x_n^{\alpha_n}$. 
Let \[h(x,w) = \sum_{0 \leq \abs{\bm{\alpha}} \leq H} \tilde{h}_{\bm{\alpha}}(w)x^{\bm{\alpha}}\] be a polynomial in $x$, where the coefficients $\tilde{h}_{\bm{\alpha}}(w)$ are continuous functions. Let $q(x)$ be a positive, continuous function on $X \subset \RR^n$, such that $\int_Xq(x)dx > 0$. Define

\[
K(w) = \int_X h^2(x,w)q(x)dx.
\]

Then $K(w)$ is contact equivalent to the sum of the squared coefficients of $h(x,w)$:

\[
K(w) \sim H(w) = \sum_{0 \leq \abs{\bm{\alpha}} \leq H} \tilde{h}^2_{\bm{\alpha}}(w).
\]    
\end{proposition}

\section{Geometry of the Newton polygon} \label{app:newton_polygon}

In this section, $\KK$ is a field of characteristic zero.
We give a short presentation of Newton polygons and Puiseux series; we closely follow \cite[Chapter~1]{casas2000singularities}, where we also refer the reader for further details. 
Consider a polynomial $f(x,y)= \sum_{\alpha,\beta=0}^\infty c_{\alpha\beta} x^\alpha y^\beta \in \KK[x,y]$, such that $d_y=\deg_yf$ and $d_x=\deg_xf$. By convention, we define the degree of the zero polynomial to be $\deg0=-\infty$. A term of $f(x,y)$ is a monomial scaled by a non-zero constant, such as $c_{\alpha\beta} x^\alpha y^\beta$. An exponent of $f(x,y)$ is a tuple $(\alpha,\beta)$ such that $c_{\alpha\beta} \neq 0$. We associate to $f$ a convex polygon, $\newton(f)$, that encodes geometric information about the exponents that occur in $f$. 

\begin{definition}(Newton Polygon)
Let $D(f)=\{(\alpha,\beta)\mid c_{\alpha\beta}\neq0\} \subseteq \NN^2$ be the Newton diagram of $f(x,y)$.
We obtain the Newton polygon, $\newton(f)$, of $f$ by attaching a copy of the positive quadrant $\RR_{\geq 0}^2$ to each point of $D(f)$ and considering the convex hull of the union. In particular,
we have 
\[
\newton(f) \coloneqq \text{Conv}( D(f)\oplus \RR_{\geq0}^2) .
\]
\end{definition}

\begin{comment}
Let $\alpha^* \in \NN$ be the largest integer such that $x^\alpha$ divides $f(x,y)$ and let $\beta \in\NN$ be the largest integer such that $y^\beta$ divides $f(x,y)$. 
\end{comment}

 Let $\bm{p}_i = (\alpha_i,\beta_i), i\in[K+1]$ be the vertices of $\newton(f)$, ordered from left to right. Let $\Delta_j$, for $0 \leq j \leq K+1$, denote the facets of $\newton(f)$, ordered from left to right; then, $\Delta_0$ and $\Delta_{K+1}$ are the two non-compact facets (in our case, half-lines) of $\newton(f)$. 
 Let $\Delta_i = \left(\bm{p}_i,\bm{p}_{i+1}\right), i\in[K]$ be the compact facets (in our case, segments) of $\newton(f)$ with vertices $\bm{p}_i$ and $\bm{p}_{i+1}$. 
For a compact facet $\Delta_i$, the height and the width  are  $h_i=\beta_{i-1}-\beta_i$ and $w_i=\alpha_i -\alpha_{i-1}$, respectively. The slope of $\Delta_i$ is given by $s_i = - \frac{q_i}{p_i}$, for $p_i$ and $q_i$ two co-prime positive integers. 
We call $wt_i=(p_i,q_i)$ the weight of a compact facet $\Delta_i$. The length of a compact facet $\Delta_i$ is the integer $l_i=\frac{h_i}{q_i}=\frac{w_i}{p_i}$. For the non-compact facets, we set $l_0=l_{K+1}=\infty$. Finally, let $\bar{\alpha}$ be the largest integer such that $x^{\bar{\alpha}}$ divides $f(x,y)$ and $\bar{\beta}$ be the largest integer such that $y^{\bar{\beta}}$ divides $f(x,y)$. Then, the height and the width of the Newton polygon of $f$ are $h_f=  \beta_0 - \bar{\beta}$ and $w_f = \alpha_k - \bar{\alpha}$, respectively. An illustration of the Newton polygon of a given polynomial is given in Figure~\ref{fig:newton_polygon}.

\begin{figure}[htp]
    \centering
    \includegraphics[width=7cm]{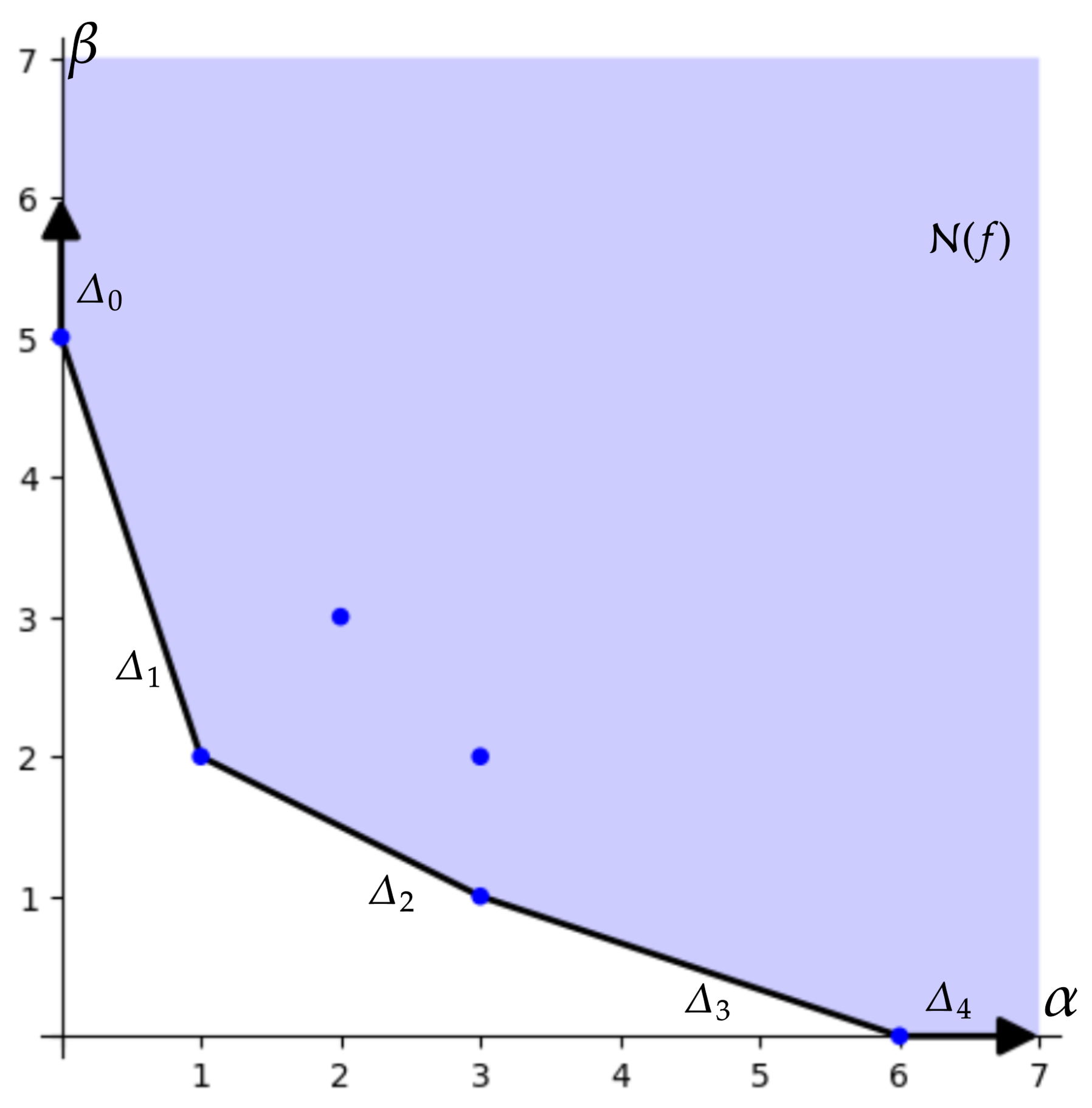} 
    \caption{Newton polygon $\newton(f)$ of $f(x,y) = y^5 + y^2x + yx^3+x^6+y^2x^3+y^3x^2$, with $K=3$ compact faces.}
    \label{fig:newton_polygon}
\end{figure}

For each compact facet $\Delta_i$ of weight $wt_i=(p_i,q_i)$, we can consider the restriction of $f(x,y)$ to this facet:
    \begin{equation}
    f_{\Delta_i}(x,y) \coloneqq \sum_{(\alpha,\beta) \in D(f) \cap \Delta_i} c_{\alpha\beta}x^\alpha y^\beta \in \KK[x,y].
    \end{equation}

The exponents $(\alpha,\beta)$ of the restriction $f_{\Delta_i}$, 
all lie on the lattice points of the segment supporting $\Delta_i$. Hence, we have that $\alpha\, p_i + \beta \, q_i = k_i$, 
for some integer $k_i$, where $-\tfrac{q_i}{p_i}$ is the slope of $\Delta_i$. In other words, it is a $(p_i,q_i)$-weighted homogeneous polynomial. Therefore, $f_{\Delta_i}$  corresponds to a univariate polynomial, which we call the \textit{facet polynomial} of $\Delta_i$:

\begin{definition}[Facet polynomial]\label{def:ufp}
    Let $f(x,y) \in \KK[x,y]$, with Newton polygon $\newton(f)$. Take a compact facet $\Delta_i$ of weight $(p_i,q_i)$ and length $l_i$. Consider the restriction $f_{\Delta_i}(x,y) = \sum_{(\alpha,\beta) \in D(f) \cap \Delta_i} c_{\alpha\beta}x^\alpha y^\beta$ as above. Then, $$f_{\Delta_i}(x,y) = x^{\alpha_{i+1}}y^{\beta_{i+1}}\sum_{t= 0}^{l_i} c_t (x^{-p_i}y^{q_i})^t,$$where $c_t = c_{\alpha\beta}$ for $(\alpha,\beta) = (\alpha_{i+1}-p_it,\beta_{i+1}+q_it)$. We call the polynomial \[F_{\Delta_i}(z)\coloneq\sum_{t=0}^{l_i}c_tz^t,\] the \textit{facet polynomial} of $\Delta_i$ in $\newton(f)$. Note that if $\Delta$ is a non-compact facet then $f_\Delta(x,y)=x^{\bar{\alpha}}a(y)$ or $f_\Delta(x,y)=y^{\bar{\beta}}b(y)$, when $\Delta$ is a vertical or horizontal facet, respectively. In this case, we let $F_\Delta(z)=z^{\bar{\alpha}}$ or $F_\Delta(z)=z^{\bar{\beta}}$, respectively.
\end{definition}

Note that, for any facet $\Delta_i$ of $\newton(f)$, $F_{\Delta_i}(z)\in\KK[z]$. Moreover, if $\Delta_i$ is compact, then $\deg(F_{\Delta_i}(z))=l_i$, the length of $\Delta_i$. Let $r_i$ be the number of distinct roots of $F_{\Delta_i}$. We denote by $a_{ij}$, for $j\in[r_{i}]$, the distinct roots of $F_{\Delta_i}$, and by $m_{ij}$, for $j\in[r_{i}]$, their respective multiplicities.

\begin{definition}[Main face]
    Let $\mathcal{D} = \{(\alpha,\beta) \in \RR^2_{\geq 0}\mid \alpha=\beta \}$ be the diagonal of the quadrant $\RR^2_{\geq 0}$. The \textit{main face} of $f$ is the face of the boundary $\partial\newton(f)$ of the Newton polygon, intersected by $\mathcal{D}$.
\end{definition}

\begin{definition}[Newton distance]
    Let $\bm{p} = (\delta,\delta) \in \QQ_{\geq0}^2$ be the point of intersection of $\mathcal{D}$ with $\partial\newton(f)$. Then, $\delta_f \coloneq \delta \in \QQ_{\geq0}$ is the \textit{Newton distance} of $f$.
\end{definition}

For any $f(x,y) \in \KK[x,y]$, the main face of $f$ is either a vertex, a compact facet or a non-compact facet;
the latter are a segment or a half-line, respectively. In Fig.~\ref{fig:Netwon-polygon-w-diagonal}, $\Delta_2$ is the main face; in this case a segment.
The Newton distance is $\tfrac{5}{3}$.

\begin{figure}[htp]
    \centering
    \includegraphics[width=7cm]{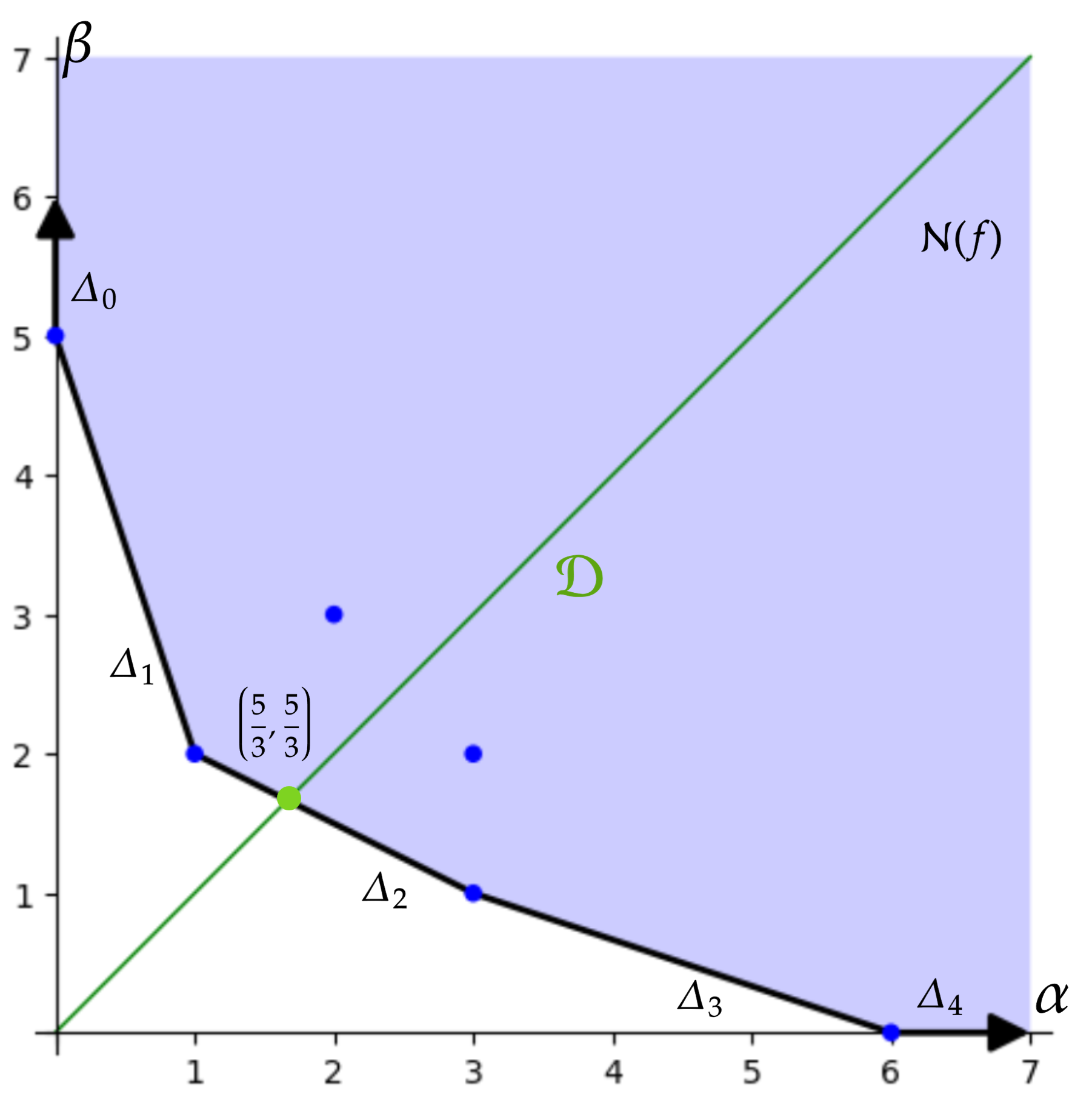} 
    \caption{Newton polygon $\newton(f)$ of $f(x,y) = y^5 + y^2x + yx^3+x^6+y^2x^3+y^3x^2$, with $K=3$ compact faces. The main face of $\newton(f)$ is $\Delta_2$, and the Newton distance of $f$ is $\delta_f = \frac{5}{3}$.}
    \label{fig:Netwon-polygon-w-diagonal}
\end{figure}

If a polynomial factors, then its Newton polygon is related to the Newton polygons of its factors, through the Minkowski sum operation.
The following classical lemma illustrate this.

\begin{lemma}\label{lem:minkowski}
    Let $f_1,f_2$ and $f$ be polynomials in $\KK[x,y]$ such that $f=f_1f_2$. Then, $\newton(f)=\newton(f_1) \oplus \newton(f_2)$. Furthermore, $\newton(f) \subseteq \newton(f_i)$ and $\delta_{f_i} \leq \delta_f$ for $i=1,2$, 
\end{lemma}

\begin{proof}
    This result for the \textit{convex hull} of the Newton diagram goes back to \cite{ostrowski1975multiplication}. A simple proof can be found in \cite[Lemma~2.1]{gao2001absolute}. To adapt it to the \textit{lower convex hull} case, notice that Minkowski addition is associative and commutative (from the associativity and commutativity of Euclidean vector addition). Notice also that $\RR^2_{\geq 0} \oplus \RR^2_{\geq 0} = \RR^2_{\geq 0}$ and recall that for any sets $A \subseteq \RR^n$ and $B \subseteq \RR^n$, $\text{Conv}(A \oplus B) = \text{Conv}(A)\oplus\text{Conv}(B)$. Then, we have:

    \begin{equation}
    \begin{split}
        \newton(f) &= \text{Conv}(D(f) \oplus \RR^2_{\geq 0}) \\
        &= \text{Conv}(D(f_1f_2) \oplus \RR^2_{\geq 0}) \\
        &= \text{Conv}(D(f_1f_2)) \oplus \RR^2_{\geq 0} \\
        &= (\text{Conv}(D(f_1)) \oplus \text{Conv}(D(f_2)) ) \oplus \RR^2_{\geq 0} \\
        & = (\text{Conv}(D(f_1) \oplus \RR^2_{\geq 0} )\oplus (\text{Conv}(D(f_2)  \oplus \RR^2_{\geq 0})) \\
        &= \newton(f_1)\oplus\newton(f_2).
    \end{split}
    \end{equation}

For the second part of the lemma, consider a point $\bm{p} \in \newton(f)$. Then, there exists $\bm{p_1} \in \newton(f_1)$ and $\bm{p_2} \in \newton(f_2)$ such that $\bm{p} = \bm{p_1}+ \bm{p_2}$. Since $\bm{p_i} \in \RR^2_{\geq 0}$ for $i=1,2$, we have that $\bm{p} \in \bm{p_i} \oplus \RR^2_{\geq 0} \subseteq \newton(f_i)$ for $i=1,2$. Since $\newton(f) \subseteq\newton(f_i)$, $\mathcal{D}\cap\newton(f) \subseteq \mathcal{D}\cap\newton(f_i)$, and hence $\delta_{f_i}\leq \delta_f$, for $i=1,2$.

\end{proof}

The following lemma relates the Newton distance of $f$ 
with the Newton distance of $f^m$, for some positive integer $m$.
\begin{lemma} \label{lem:dilation}
    Let $f(x,y) \in \KK[x,y]$ and $m\in\NN_{>0}$. Then, $\delta_{f^m} = m\delta_f$.
\end{lemma}

\begin{proof}
It holds 
$\newton(f^m)= \bigoplus_{i=1}^m \newton(f) = m\newton(f)$, where the first equality is from Lemma~\ref{lem:minkowski}. For the second equality, see, for example, \cite[Section~3.1]{schneider2013convex}. Suppose that $\mathcal{D}$ intersects $\partial\newton(f)$ at the point $\bm{p}=(\delta,\delta)$. Then, $\mathcal{D}$ intersects $\newton(f^m)$ at the point $m\cdot\bm{p}=(m\delta,m\delta)$. Therefore, $\delta_{f^m} = m\delta_f$ as required. 
\end{proof}

\section{Puiseux series} \label{app:yroots}

The classical theory of Newton-Puiseux expansions relates the geometry of Newton polygon to the roots of $f(x,y) \in \KK[x,y]$, when considered as a polynomial in $y$ (resp. $x$), with coefficients in $\KK[x]$ (resp. $\KK[y]$). Without loss of generality,  we consider in the sequel that polynomials in $\KK[x,y]$ are polynomials in $y$ with coefficients in $\KK[x]$. That is to say, we write $f(x,y) = \sum_{\alpha,\beta=0}^\infty (c_{\alpha\beta} x^\alpha) y^\beta \in (\mathbb{K}[x])[y]$. A $y$-root of $f(x,y)$ is an element $\phi(x)$ of the algebraic closure of $\KK[x]$, such that $f(x,\phi(x))=0$. It is not sufficient for $\phi(x)$ to be an element of the field of Laurent series $\bar{\KK}((x))$. For example, the polynomial $f(x,y) = y^2+x \in \KK[x,y]$ has two $y$-roots $\phi_{\pm}(x)=\pm \sqrt{x} = \pm x^{\frac{1}{2}}$, none of which have a Taylor expansion at the origin. Instead, $y$-roots are series which admit fractional exponents. The Puiseux Theorem \citep[Theorem 1.5.4]{casas2000singularities} asserts that the algebraic closure of $\KK((x))$ is contained in the field of Puiseux series $\puiseux{\bar{K}}{x}=\bigcup_{n\in\NN_{>0}} \mathbb{\bar{K}}((x^\frac{1}{n}))$. The field of Puiseux series is the union of all Laurent series over $\bar{\KK}$ with fractional exponents of bounded denominator. 

Let $\phi(x) = \sum_{i\geq l} b_i x^\frac{i}{n}$ be a Puiseux series. Similarly to the case of general power series, we define the order of $\phi(x)$ to be $\ord(\phi(x))=\frac{1}{n}{\min\{i \,\vert\, b_i \neq 0\}}$. If $\phi(x) \in \puiseux{\bar{\KK}}{x}$ has finitely many terms, we define the degree of $\phi(x)$ to be $\deg\phi(x)=\frac{1}{n}{\max\{i \,\vert\, b_i \neq0\}}$. By convention, we define the order of the zero series to be $\ord0=+\infty$. Equipped with the order, the field of Puiseux series is a \emph{valuation ring}.

If $\phi(x)$ is a $y$-root of $f(x,y)$ of strictly positive order, then 
we can compute its terms inductively (up to an arbitrary number of terms) by exploiting the geometry of the Newton polygon, following the Newton-Puiseux algorithm \cite[Section~1.4]{casas2000singularities}. The Newton polygon $\newton(f)$ characterizes the first order terms of the $y$-roots of strictly positive order of $f(x,y)$ in the following way. Suppose that $\newton(f)$ has $K$ compact faces. Let $\Delta_i$ be a face of $\newton(f)$ of weight $wt_i=(p_i,q_i)$ and let $a_{ij}$ be a root of $F_{\Delta_i}(z)$ of multiplicity $m_{ij}$ (see Definition \ref{def:ufp}). Let $\zeta_{q_i}$ be a $q_i$-th root of unity. Then, for all $(i, j, k) \in[K]\times[r_i]\times[q_i]$, there exist $m_{ij}$ $y$-roots $\phi(x)$ of strictly positive order, such that $\phi(x)=(\zeta_{q_i})^ka_{ij}x^{\frac{p_i}{q_i}}+\ldots$. Moreover, there exist $m_i = h_i=q_i\sum_{j=1}^{r_i}m_{ij}=q_il_i$ $y$-roots $\phi(x)$ of strictly positive order, such that $\ord(\phi)=\frac{p_i}{q_i}$. In total, there are  $m=\sum_{i=1}^K m_i=\sum_{i=1}^K h_i = h_f$ $y$-roots of strictly positive order. 

The largest denominator $n$ of the exponents of $\phi(x)$ is called the ramification index of $\phi(x)$. Let $\zeta_n$ be an $n$-th root of unity. Then, the map $\sigma_{\zeta_n}:\mathbb{\bar{K}}((x^\frac{1}{n}))\to\mathbb{\bar{K}}((x^\frac{1}{n})),\,  x^{\frac{1}{n}} \mapsto \zeta_n x^\frac{1}{n}$ is an automorphism of $\mathbb{\bar{K}}((x^\frac{1}{n}))$ over $\mathbb{\bar{K}}((x))$. Therefore, any Puiseux series $\phi(x)$ of ramification index $n$ has $n$ conjugates of the form $\sigma_{({\zeta_n})^j}(\phi(x)) = \sum_{i \geq l} b_i (\zeta_n)^{ij}x^{\frac{i}{n}}$ for $j \in [n]$.

A consequence of the Puiseux Theorem is the unique factorization of a polynomial $f(x,y)$ over $\bar{\KK}((x))$ and $\puiseux{\bar{\KK}}{x}$ \cite[Corollaries~1.5.5 and 1.5.6]{casas2000singularities}. For a set of conjugate $y$-roots $\{\phi_{ij}(x)\}_{i=1}^{n_j}$ of ramification index $n_j$, we let $f_j(x,y) = \prod_{i=1}^{n_j}(y-\phi_{ij}(x)) \in \bar{\KK}((x))[y]$. Note that $\deg_yf_j = n_j$. Let $f(x,y)$ be a polynomial of $y$-degree $d$. Let $n_1,\ldots,n_l$ be the set of ramification indices of the conjugacy classes of its $y$-roots. Then, there exists a unit $u \in \KK((x,y))$, such that $f(x,y)$ factors as:
\[
    f(x,y) = u \, x^r\prod_{j=1}^l f_j(x,y) = u \, x^r\prod_{i=1}^{d}(y-\phi_i(x)),
\]
where the first equality corresponds to irreducible elements of $\bar{\KK}((x))[y]$, and the second to the linear elements of $\puiseux{\KK}{x}[y]$. 
In addition, the sum of the ramification indices of the $y$-roots of $f$ equals the $y$-degree of $f$, that is to say,
$n_1 + \cdots+n_l=d_y$.

\subsection{Truncated Puiseux series}
For our purposes, we need to introduce a truncation of Puiseux series, which plays a central role in the analysis of Puiseux expansions of polynomials \citep{walsh2000polynomial, poteaux2021computing}. We use notation from \cite{poteaux2021computing}.

\begin{definition}[Truncation of Puiseux series]
Let $\phi(x) = \sum_{i=k}^\infty b_i x^\frac{i}{n} \in \puiseux{\bar{K}}{x}$ be a Puiseux series and let $\tau \geq k$. Then, the \textit{truncation} of $\phi(x)$ at $\tau \neq \infty$ is $\trunc{\phi(x)}{\tau} \coloneq \sum_{i \leq \tau} b_i x^\frac{i}{n}$.
If   $\tau=\infty$, then $\trunc{\phi(x)}{\infty} = \phi(x)$.  
\end{definition}
The singular part of a $y$-root $\phi(x)$ of a polynomial $f(x,y)$ is a truncation that contains the necessary information to characterize the singular locus at the origin. The computation of singular parts of the $y$-roots and its complexity have been discussed in several articles, e.g. \cite{duval1989rational, walsh2000polynomial,poteaux2021computing}.

\begin{definition}[Generalized multiplicity {\cite[Proof of Theorem 5]{phong1999growth}}]
Let $f(x,y) \in \KK[x,y]$ be a polynomial of degree $\deg_yf=d$. Let $\phi(x) \in \puiseux{\bar{K}}{x}$. Let $R_f = \{\phi_i(x)\}_{i=1}^d$ be the set of all (not necessarily distinct) $y$-roots of $f$, counted with multiplicity. 
We say that $\phi(x)$ has generalized multiplicity $e \leq d$ with respect to $R_f$,
if there exists exactly $e$ $y$-roots $\{\phi_{i_1},\ldots,\phi_{i_e}\} \subseteq R_f$  such that $\phi(x) = \trunc{\phi_{i_1}(x)}{\tau}=\ldots=\trunc{\phi_{i_e}(x)}{\tau}$ for some $\tau \in \mathbb{Z} \cup \{\mathbb{\infty}\}$. 
\end{definition}

Note that if $\tau = \infty$ then, $f(x,y)$ has a $y$-root $\phi(x)$ of multiplicity $m_\phi = e$.

\begin{definition}[Singular part of a $y$-root {\cite[Section~2]{walsh2000polynomial}}] \label{def:singpart}
Let $f(x,y) \in \KK[x,y]$. Let $\phi(x) \in R_f$ be a $y$-root of $f(x,y)$. Then, the \textit{singular part} $S_\phi (x)$ of $\phi(x)$ is the truncation 
\[
S_\phi (x)=\trunc{\phi(x)}{{\tau^*}}, 
\]
where $\tau^*$ is such that 
\[
\tau^* = \max\{\tau \in\mathbb{Z} \mid \exists \tilde{\phi}(x)\neq\phi(x) \in R_f \, \text{s.t.} \, \trunc{\tilde{\phi}(x)}{\tau}=\trunc{\phi(x)}{\tau}\}.
\]
We call  $\tau^*$ the regularity index of $\phi(x)$.
\end{definition}

In other words, the regularity index $\tau^*$ is the smallest truncation order which distinguishes $\phi(x)$ from the other distinct $y$-roots of $f(x,y)$.

\begin{example}
\label{ex:singparts}
    Singular parts of $y$-roots can have fractional exponents. As an example, consider the Puiseux series $\phi(x)=x+x^{\frac{3}{2}}+x^{\frac{7}{4}}$. The $4$-th roots of unity are $\zeta^0=1,\zeta^1=i,\zeta^2=-1$ and $\zeta^3=-i$. We let $\phi_i=\sigma_{\zeta^i}(\phi(x))$. Then, $f(x,y) = \prod_{i=0}^3(y-\phi_i(x))$ is a polynomial in $\KK[x,y]$. The singular parts of this polynomial are $S_{\phi_0}(x)=S_{\phi_3}(x)=x+x^{\frac{3}{2}}$ and $S_{\phi_1}(x)=S_{\phi_2}(x)=x-x^\frac{3}{2}$, and their regularity indices are equal to $6$.
\end{example}

We introduce a further subdivision of the singular part, which only retains the initial terms of the singular part with integer exponents.

\begin{definition}[Polynomial part of a $y$-root]
    \label{def:poly-part}
Let $\phi(x)$ be a $y$-root of $f(x,y)$ and let $S_\phi$ be its singular part. The polynomial part $P_\phi(x)$ of $\phi(x)$ is the largest polynomial contained in its singular part. That is 
\[
P_\phi(x) = \trunc{S_\phi(x)}{d} ,
\]
where $d = \max\{\tau \in \ZZ \mid \trunc{S_\phi(x)}{\tau} \in \KK[x,y]\}$ is the degree of the polynomial part of $P_\phi(x)$. If $d=0$, then $P_\phi(x)$ is the zero polynomial.
\end{definition}

\begin{example}
The polynomials parts of the $y$-roots of $f(x,y)$ in Example \ref{ex:singparts} are $P_{\phi_i}(x)=x$ for $0 \leq i \leq 3$.
\end{example}

\section{Deferred Proofs of Subsection~\ref{subsec:newton_polygons_rlcts}}
\label{sec:deff_proof}

\subsection{Proof of Proposition~\ref{prop:main}}
\label{subsec:proof_prop:main}

\begin{proof}

We follow the proof of \cite{collins2018log} and the proof of \cite{phong1999growth}. Let $f(x,y)$ be an analytic function. Suppose that $\Delta_i$ is to the \textit{right} of the main face. Let $a_{ij}$ be a root of $F_{\Delta_i}(z)$ of multiplicity $m_{ij}$. Then, by Section \ref{app:yroots}, we have that $\delta_f \geq h_i \geq m_{ij}$. Now suppose that $\phi(x) = a_{ij} x^{\frac{p_i}{q_i}}+\ldots $ is a $y$-root of $f(x,y)$, such that $a_{ij}$ is a root of multiplicity $m_{ij}$ of $F_{\Delta_i}(z)$, and that $p_i$ and $q_i$ are coprime integers. Let $\zeta_{q_i}$ be a $q_i$-th root of unity. Then, all the $q_i$ conjugates $\sigma_{\zeta_{q_i}^0}(\phi(x)),\ldots,\sigma_{\zeta_{q_i}^{q_i -1}}(\phi(x))$ are also $y$-roots of $f(x,y)$. Therefore, $m_{ij} \leq \frac{h_i}{q_i}$. Furthermore, $\frac{h_i}{w_i}=\frac{q_i}{p_i}$. This implies that $h_i \geq q_i  m_{ij}$ and $w_i \geq p_i  m_{ij}$. Now if $\Delta_i$ is to the left of the main face, we have that $\delta_f \geq w_i \geq m_{ij}$. Therefore, facets to the left and to the right of the main face satisfy the normalization condition.

If the main face is a vertex, we are done - since all facets satisfy the normalization condition. Moreover, if the main face is a non-compact facet $\Delta$, with facet polynomial $F_\Delta(z)=z^m$, for some $m\in\NN_{>0}$. But $\delta_f=m$, and so $\Delta$ is normalized. Suppose now that the main face is a compact facet. Let $\Delta_i$ be this facet and let $s_i = - \frac{q_i}{p_i}$ be its slope. Then, the equation of the line cutting out $\Delta_i$ is given by $$L(\alpha,\beta) = \frac{q_i\alpha+p_i\beta}{p_i + q_i}-\delta_f =0.$$ Note that $L(0,h_i) = \frac{h_ip_i}{p_i+q_i}-\delta_f \leq \delta_f(\frac{p_i}{p_i+q_i}-1)\leq 0$. Hence, 

\begin{equation}\label{eq:newt_dist_lower_bound}
\delta_f \geq \frac{h_i p_i}{p_i + q_i} \geq m_{ij} \frac{p_iq_i}{p_i+q_i}.
\end{equation}

If $m_{ij} > \delta_f$ then, $p_i + q_i > p_iq_i$. If $p_i >1$ and $q_i > 1$, then $p_i+ q_i < 2\max(p_i,q_i) \leq p_iq_i$. Therefore, either $p_i$ or $q_i$ is equal to $1$. Without loss of generality, we assume that $q_i = 1$. Recalling from Section \ref{app:yroots} that $m_{ij}$ corresponds to the number of $y$-roots of $f(x,y)$ with a given initial term $a_{ij} x^{\frac{p_i}{q_i}}$ and since $q_i=1$, we have that $f(x,y)$ has $m_{ij} > \delta_f$ $y$-roots of the form $\phi(x) = a_{ij} x^{p_i} + \ldots$.

Let $P_f(x) \in \mathbb{K}[[x]]$ be a power series such that $P_f(x)$ has initial term $a_{ij} x^{p_i}$. Let $P_f(x)$ of maximal degree (possibly infinite) such that the generalized multiplicity of $P_f(x)$ is $e > \delta_f$. Then, consider the right equivalence $\Phi:(x,y) \mapsto(x,y-P_f(x))$. Let $\tilde{f} = f\circ\Phi$. 

Then, \[\partial \newton(f) \cap \partial \newton(\tilde{f}) = \{(\alpha,\beta) \mid (\alpha,\beta)\in\partial \newton (f) \, \text{s.t.} \, \beta \geq e \}. \] If $P_f(x)$ lies in the set $\KK[[x]]\setminus \KK[x]$ then, $\phi(x)=0$ is a $y$-root of multiplicity $e$ of $\tilde{f}$. Therefore, $\delta_{\tilde{f}}=e$, since $\mathcal{D}$ intersects a non-compact horizontal facet of $\partial \newton(\tilde{f})$. 

Otherwise, $P_f(x)$ is polynomial, and the main facet $\tilde{\Delta}_i$ of $\tilde{f}$ has a finite slope $\tilde{s}_i < 0$. If $\tilde{s}_i = -\frac{\tilde{q}_i}{\tilde{p}_i}$ with $\tilde{p}_i\geq2$ and $\tilde{q}_i\geq2$ then, $\tilde{F}_{\tilde{\Delta}_i}$ has no root of multiplicity $\tilde{m}_{ij}$ greater than $\delta_{\tilde{f}}$. Otherwise, if, without loss of generality,  $\tilde{s}_i= - \frac{1}{\tilde{p}_i}$ and $\tilde{f}_{\Delta_i}$ has a factor of multiplicity greater than $\delta_{\tilde{f}}$, the maximality assumption on the degree of $P_f(x)$ is contradicted. Therefore, $\tilde{f}$ is normalized, and $\tilde{f} = f \circ \Phi$ is right equivalent to $f$. 

\end{proof}

\begin{corollary}\label{cor:multiplicity_and_rlct}
    Let $f(x,y)\in\KK[x,y]$ and let $\phi(x)\in\KK[[x]]\setminus\KK[x]$ be a $y$-root of $f(x,y)$ of multiplicity $m > \delta_f$. Then, $\rlct_0(f) = \frac{1}{m}$.
\end{corollary}

\begin{proof}
    This follows from the last part of the proof of Proposition~\ref{prop:main}.
\end{proof}

\section{Deferred Proofs of Main Results }\label{sec:proof_of_main}

\subsection{Proof of Proposition~\ref{prop:unique}}
\label{subsec:proof_lem:unique}
First, we show two auxiliary lemmas used in the proof of Proposition~\ref{prop:unique}.
\begin{lemma}\label{lem:maxmult}
    Suppose that $f(x,y) \in \KK\{x,y\}$ is an unnormalized analytic function. Let $\Delta_i$ be its main face and $m_{ij}$ be the multiplicity of some root of the univariate facet polynomial $F_{\Delta_i}(z)$, such that $m_{ij} > \delta_f$. Let $m_i = \deg F_{\Delta_i}(z)$. Then, $m_{ij} > \frac{m_i}{2}$.
\end{lemma}

\begin{proof}
    Let $f(x,y) \in \KK\{x,y\}$ be any unnormalized analytic function. Then, $\newton(f)$ has a main face $\Delta_i$ of slope $s_i = -\frac{1}{p_i}$, such that the facet polynomial $F_{\Delta_i}(z)$ has a root of multiplicity $m_{ij} > \delta_f$. We have seen in the course of the proof of Proposition \ref{prop:main}, Equation \ref{eq:newt_dist_lower_bound}, that $\delta_f \geq \frac{p_im_i}{1+p_i}$. Since $p_i \geq 2$, we have that $m_{ij} > \delta_f\geq \frac{m_i}{2}$ which implies that $m_{ij} > \frac{m_i}{2}$. 
\end{proof}

\begin{lemma}\label{lem:rational}
    If $f(x,y)\in\QQ\{x,y\}$, then the normalizing power series $P_f(x)$ of $f(x,y)$ lies in $\mathbb{Q}[[x]]$.
\end{lemma}

\begin{proof}
    Suppose that $f(x,y) \in \QQ\{x,y\}$ and is unnormalized. Suppose that $P_f(x) = ax^p$. Then, the main face of $\newton(f)$ being $\Delta_i$, of weight $wt_i=(p_i,1)$, we have $p=p_i$ and $a=a_{ij}$, a root of the facet polynomial $F_{\Delta_i}(z)$. 
    Then, the generalized multiplicity of $P_f(x)$ is $e=m_{ij}$, the multiplicity of $a_{ij}$ as a root of $F_{\Delta_i}(z) \in \QQ[z]$. Suppose that $a_{ij}$ is algebraic over $\QQ$. Then, there exists at least one algebraic conjugate $a_{ij'}$ of $a_{ij}$, such that $a_{ij'}$ is a root of $F_{\Delta_i}(z)$. Moreover, $m_{ij'}=m_{ij}$. But, by Lemma \ref{lem:maxmult}, we have that $m_{ij}+m_{ij'}>m_i$, which is a contradiction. Therefore, $a=a_{ij}$ is not algebraic over $\QQ$ and $P_f(x)\in\QQ[[x]]$. The statement follows by induction on the number of terms of $P_f(x)$.
\end{proof}

Now we proceed to the proof of Proposition~\ref{prop:unique}.
\begin{proof}
    Suppose that $f(x,y) \in \KK\{x,y\}$ is an unnormalized analytic function, and consider $P_f(x) \in \KK\{x\}$ such that $\tilde{f}(x,y) = f(x,y-P_f(x))$ is normalized. Suppose that there exists a power series $P_f'(x) \neq P_f(x)$ whose generalized multiplicity $e'$ is greater than or equal to the generalized multiplicity $e$ of $P_f(x)$ and whose degree is equal to that of $P_f(x)$. Then, there exists a truncation $P_f''(x)$ of $P_f(x)$ and $P_f'(x)$ such that $\tilde{f}(x,y) \coloneqq f(x,y-P_f''(x))$ is an unnormalized power series, with main face $\tilde{\Delta}_i$. 
    Let $\tilde{F}_{\tilde{\Delta}_i}(z)$ be the facet polynomial associated to $\tilde{\Delta}_i$. Let $\tilde{m}_i$ be its degree, and let $\delta_{\tilde{f}}$ be the Newton distance of $\newton(f)$. Then $\tilde{F}_{\tilde{\Delta}_i}(z)$ has two roots of multiplicity $e$ and $e'$, such that $e' \geq e >\delta_{\tilde{f}}$. By definition, $e+e'\leq \tilde{m}_i$. But by Lemma \ref{lem:maxmult}, $e+e' > \frac{\tilde{m}_i}{2}+\frac{\tilde{m}_i}{2}=\tilde{m}_i$, which is a contradiction. Hence, $P_f(x)$ must be unique. Rationality of $P_f(x)$ follows from Lemma~\ref{lem:rational}.
\end{proof}

\begin{remark}
    By Lemma \ref{lem:rational}, if we seek to compute the local RLCT of a polynomial $f(x,y)$ with rational coefficients, we do not need to work over algebraic field extensions of $\QQ$. At each iteration of the algorithm, we only need to care about the unique (by Proposition~\ref{prop:unique}) rational root of $F_{\Delta_i}(z)$ of multiplicity $m_{ij} > \delta_f$. This means that we do not need a full factorization of $F_{\Delta_i}(z)$ at each step, merely a square-free factorization.
\end{remark}

\subsection{Proof of Theorem~\ref{thm:finite_part_main}}
\label{subsec:proof_thm:finite_part_main}

Since $P_f(x)\in\QQ[[x]]$, it follows that $\finitepart{f}{x} \in \QQ[x]$. The following Lemma explains how to compute the local RLCT at the origin of $f(x,y)$ using a right equivalence defined in terms of $\finitepart{f}{x}$.

\begin{proof}
    Let $f(x,y)\in\QQ[x,y]$, with normalizing power series $P_f(x) \in \QQ[x]$. Let $\finitepart{f}{x}$ be as above, and construct $\tilde{f}(x,y)=f(x,y-\finitepart{f}{x})$. If $\tilde{f}$ is normalized, then $\finitepart{f}{x}=P_f(x)$, by Proposition~\ref{prop:unique}. Then, by Proposition~\ref{prop:normalized}, we have that $\rlct_0(f)=\frac{1}{\delta_{\tilde{f}}}$. Otherwise, $\tilde{f}$ is not normalized. Then, it must be that $P_f(x)$ is a $y$-root of $f(x,y)$ and $\finitepart{f}{x}$ is its singular part. Let $\Delta$ be the main face of $\newton(\tilde{f)}$. Let $\tilde{s}=\frac{-1}{p}$ be its slope. Let $a\in\QQ$ be the root of $\tilde{F}_{\Delta}$, such that the multiplicity $m$ of $a$ is maximal, and $m>\delta_f$. By the maximality assumption on the degree of singular parts (see Definition~\ref{def:singpart}), there exists a $y$-root $\tilde{\phi}(x)$ of $\tilde{f}$ of multiplicity $m$, such that the initial term of $\tilde{\phi}(x)$ is $ax^p$. Then $P_f(x)=\finitepart{f}{x}+\tilde{\phi}(x)$ is the normalizing power series of $f(x,y)$ and it corresponds to a $y$-root of $f(x,y)$ of order $m$. Therefore, $\rlct_0(f)=\frac{1}{m}$.
\end{proof}

\subsection{Proof of Theorem~\ref{thm:complexity_main}}
\label{subsec:proof_thm:complexity_main}
An upper bound on the degree $d$ of the finite part $\finitepart{f}{x}$ of the normalizing power series $P_f(x)$ of $f(x,y)$, is an upper bound on the number of iterations of Algorithm~\ref{alg:RLCT_main}. Consider two $y$-roots $\phi_1(x) = \finitepart{f}{x}+\tilde{\phi}_1(x)$ and $\phi_2(x)=\finitepart{f}{x}+\tilde{\phi}_2(x)$, such that $\tilde{\phi}_1(x)$ and $\tilde{\phi}_2(x)$ are not necessarily distinct Puiseux series and $\ord \tilde{\phi}_i(x)>d$, for $i=1,2$. Then, $\ord (\phi_1(x)-\phi_2(x))=\min\{\ord\tilde{\phi}_1(x),\ord\tilde{\phi}_2(x)\}>d$. Therefore, finding an upper bound on the order of the differences of $y$-roots of $f(x,y)$ provides an upper bound on $d$. A natural approach is to consider the discriminant in $y$ of $f(x,y)$.

Before deriving upper bounds on the complexity of our algorithms, we remind the reader of  some basic notions. We recall the definitions and basic properties of resultants and discriminants, as well as the definition of the square-free factorization of a polynomial. A classical reference for the former can be found in \cite{cox1997ideals}, Chapter 3. 

\begin{definition}[Sylvester matrix, {\cite[Chapter 3, Section 6, Definition 2]{cox1997ideals}}]
Let $R$ be an integral domain. Let $f(z) = a_mz^m+\ldots+ a_1z+a_0$ and $g(z) = b_nz^n+ \ldots + b_1 z + b_0$ be two polynomials in $R[z]$ of degree $m$ and $n$ respectively, such that $a_m \neq 0$ and $b_n \neq 0$. The Sylvester matrix $\text{Syl}(f,g)$ of $f$ and $g$ is the $(m+n)\times(m+n)$ matrix defined as
$$
\text{Syl}(f,g) = 
\begin{pmatrix}
    a_0 & a_1 & \cdots & a_{m-1} & a_m & 0 & \cdots & 0 \\
    0 & a_0 & a_1 & \cdots & a_{m-1} & a_m & \cdots & 0 \\
    \vdots &&\ddots&\ddots && \ddots & \ddots & \vdots \\
    0 & \cdots & 0 & a_0 & a_1 & \cdots & a_{m-1} & a_m \\
    b_0 & b_1 & \cdots & b_{n-1} & b_n & 0 & \cdots & 0 \\
    0 & b_0 & b_1 & \cdots & b_{n-1} & b_n & \cdots & 0 \\
    \vdots &&\ddots&\ddots && \ddots & \ddots & \vdots \\
    0 & \cdots & 0 & b_0 & b_1 & \cdots & b_{n-1} & b_n \\
\end{pmatrix} \in R^{(m+n)\times(m+n)}.
$$
\end{definition}

\begin{definition}[Resultant, {\cite[Chapter 3, Section 6, Definition 2]{cox1997ideals}}] \label{def:res}
The resultant $\text{res}(f,g)$ of $f$ and $g$ is the determinant of their Sylvester matrix: $$\text{res}(f,g)=\det (\text{Syl}(f,g)).$$    
\end{definition}

\begin{proposition}\label{prop:diffroots}
Let $f(z)$ and $g(z)$ be as above. Then, the resultant of $f$ and $g$ can be expressed as: \[\text{res}(f,g) = a_m^n b_n^m\prod_{i,j} (x_i-y_j)\] where $x_1,\ldots,x_m$ and $y_1,\ldots,y_n$ are the roots of $f$ and $g$ respectively, counted with multiplicity, lying in the algebraic closure $\bar{F}$ of $F=\text{Frac}(R)$.
\end{proposition}

\begin{proof}
    See \cite[Chapter 3, Section 1]{cox1998using}.
\end{proof}

Note that, by Proposition~\ref{prop:diffroots}, $f$ and $g$ have a common root if and only if $\res(f,g)=0$.

\begin{definition}[Discriminant {\cite[Chapter 3, Section 6, Exercise 16]{cox1998using}}]\label{def:disc}
    The discriminant $\text{disc}(f)$ is defined as $$\text{disc}(f) = (-1)^{\frac{n(n-1)}{2}}\frac{1}{a_n} \text{res}(f,f')=(-1)^{\frac{n(n-1)}{2}} a_n^{2n-2} \prod_{i<j} (x_i - x_j)^2,$$ where $x_1, \ldots x_n$ are roots of $f(z)$.
\end{definition}

Note that $f$ has a multiple root if and only if $\disc(f)=0$.

We give an upper bound on the number of iterations of Algorithm \ref{alg:RLCT_main} given a polynomial $f(x,y) \in \QQ[x,y]$, by bounding the degree $d$ of $\finitepart{f}{x}\in\QQ[x]$. In the following, we denote by $\abs{\cdot} = e^{-\ord(\cdot)}:\puiseux{C}{x} \mapsto \RR_{>0}$ the \emph{absolute value induced by the valuation} $\ord(\cdot)$ on $\puiseux{C}{x}$. 

\begin{proposition}[Root separation for rational bivariate polynomials]\label{prop:complexity}
    Let $f(x,y) = a_{d_y}(x)y^{d_y}+\ldots+a_1(x)y+a_0(x)\in\QQ[x,y]$ be a square-free polynomial, such that $\deg_yf=d_y$ and $\deg_xf=d_x$. Let $0<\abs{\phi_1}\leq \abs{\phi_2}\leq\ldots\leq\abs{\phi_{d_y}}$ be the $y$-roots of $f(x,y)$. Let $\Omega$ be any set of $k$ pair of indices $(i,j)$, such that $1\leq i < j\leq d_y$. Then,

    \[  \sum_{(i,j)\in\Omega} \ord({\phi_i - \phi_j}) \leq (d_y+k-\frac{1}{2})d_x. \]
    
\end{proposition}

\begin{proof}
We adapt the proof of \cite[Theorem 7]{tsigaridas2008complexity}, from univariate polynomial with integer coefficients to bivariate polynomials with rational coefficients. Consider the multiset $\bar{\Omega} = \{j\mid (i,j) \in \Omega\}$ of cardinality $\lvert \bar{\Omega} \rvert =k$. We denote the order of the leading term of $f(x,y)$ by $h=\ord(a_{d_y}(x))$. We begin by observing that, for any $i \in [d_y]$,

\begin{equation}\label{eq:root_bounds}
    -h\leq\ord(\phi_i) \leq d_x \iff e^{-d_x} \leq \abs{\phi_i} \leq e^h. 
\end{equation}

Next, we obtain a bound on the absolute value of the product of the roots of $f(x,y)$. Note that 

\[\abs{\phi_i} \geq 1 \iff \ord \phi_i \leq 0.\]

Let $d_y^- \leq d_y$ be the number of $y$-roots of order $\leq 0$, and denote by $\{\phi_i^-\}_{i=1}^{d^-_y}$ the roots of negative order. We have the following relation between $h$ and the $y$-roots of negative order,

\[ -h = \sum_{i=1}^{d_y^-} \ord \phi_i^- \iff e^{h} = \exp({-\sum_{i=1}^{d_y^-} \ord \phi_i^-})=\prod_{i=1}^{d^-_y} \abs{\phi_i^-}.\]

Thus,

\begin{equation} \label{eq:product_all_roots}
    \prod_{i=1}^{d_y} \abs{\phi_i} \leq \prod_{i=1}^{d_y} \max\{\abs{\phi_i},1\} = \prod_{i=1}^{d^-_y} \abs{\phi_i^-} = e^h.
\end{equation}

Furthermore, since the absolute value $\abs{\cdot}$ is induced by a valuation, it satisfies the ultrametric inequality:

\begin{equation} \label{eq:ultrametric}
    \forall \phi,\phi'\in\puiseux{C}{x}, \,\, \abs{\phi-\phi'} \leq \max \{ \abs{\phi},\abs{\phi'}\}.
\end{equation}

Recall, by Definition~\ref{def:disc}, that $\disc f =(-1)^{\frac{d_y(d_y-1)}{2}} a_{d_y}(x)^{2d_y-2} \prod_{i<j} (\phi_i - \phi_j)^2$. Therefore, 

\begin{equation*}
\begin{split}
    \abs{\disc f} &= (e^{-h})^{2d_y-2}\prod_{i<j}\abs{\phi_i-\phi_j}^2 \\
                  &= (e^{-h})^{2d_y-2}\prod_{(i,j)\in\Omega}\abs{\phi_i-\phi_j}^2 \prod_{(i,j) \notin \Omega}\abs{\phi_i-\phi_j}^2.
\end{split}
\end{equation*}

We consider the product $\prod_{(i,j)\notin\Omega} \abs{\phi_i-\phi_j}$ and apply $\binom{d_y}{2}-k$ times the ultrametric inequality~\ref{eq:ultrametric}. Thus,

\begin{equation*}
    \begin{split}
        \prod_{(i,j)\notin\Omega} \abs{\phi_i-\phi_j} &\leq \abs{\phi_1}^0\abs{\phi_2}^1\cdots\abs{\phi_{d_y}}^{d_y-1} (\prod_{j\in\bar{\Omega}} \abs{\phi_j})^{-1} \\
        &\leq (\prod_{i=1}^{d_y^-}\abs{\phi^-_i})^{d_y-1} (\prod_{j\in\bar{\Omega}} \abs{\phi_j})^{-1} \\
        &\leq e^{h(d_y-1)}e^{kd_x},
        \end{split}
\end{equation*}

where we used Inequalities~\ref{eq:product_all_roots} and \ref{eq:root_bounds}. Then,
\begin{equation*}
    \begin{split}
        \prod_{(i,j)\in\Omega} \abs{\phi_i-\phi_j} &= e^{h(d_y-1)}(\prod_{(i,j)\notin \Omega} \abs{\phi_i-\phi_j})^{-1}\sqrt{\abs{\disc(f)}} \\
        & \geq e^{-kd_x}\sqrt{\abs{\disc(f)}}    
    \end{split} 
\end{equation*}

Now, by expressing $\disc f$ in terms of the determinant of the $(2d_y-1)\times(2d_y-1)$ matrix $\Syl(f,f')$, we obtain the following upper bound on the order of $\disc f$:

\begin{equation*} 
    \ord(\disc f) \leq \deg_x(\disc f) \leq (2d_y-1)d_x.
\end{equation*}

Altogether,

\[  \prod_{(i,j)\in\Omega} \abs{\phi_i-\phi_j} \geq e^{-kd_x} e^{-(d_y-\frac{1}{2})d_x} = e^{-(d_y+k-\frac{1}{2})d_x}.\]

Finally, we get:

\[ \sum_{(i,j)\in\Omega} \ord(\phi_i-\phi_j) \leq (d_y+k-\frac{1}{2})d_x.\]

\end{proof}

Now we can prove Theorem~\ref{thm:complexity_main}.
\begin{proof}
Given a polynomial $f$, consider its square-free factorization, $\tilde{f}$. This is a polynomial of degree $\deg_y \tilde{f} \leq d_y$ and $\deg_x \tilde{f} \leq d_x$. Since taking the square-free part of $f$ does not affect the finite part of its normalizing power series, $\finitepart{{\tilde{f}}}{x}=\finitepart{f}{x}$. If $\tilde{f}$ has a single $y$-root of strictly positive order, then $\finitepart{{\tilde{f}}}{x}=0$, therefore $\deg\finitepart{{\tilde{f}}}{x} = -\infty$ and the corollary follows trivially. Now suppose that $\finitepart{{\tilde{f}}}{x}$ has generalized multiplicity $e\geq 2$. Let $d = \deg \finitepart{{\tilde{f}}}{x}$. Then, there exist $\phi_{i_1}(x),\ldots,\phi_{i_e}(x)$ $y$-roots such that $\trunc{\phi_{i_j}(x)}{d} = \trunc{\phi_{i_{j'}}(x)}{d}$, for all $(j,j') \in [e]^2$. Since $\tilde{f}$ is square-free, at least $2$ of the $e$ $y$-roots are distinct. Let $\phi_{i_j}(x)$ and $\phi_{i_{j'}}(x)$ be such roots. Then, $(d_y + \frac{1}{2})d_x\geq\ord(\phi_{i_j}-\phi_{i_{j'}})>d$, by Proposition~\ref{prop:complexity}.
\end{proof}

\section{Deferred Proofs of Algorithmic Results}\label{sec:correctness}

Consider the while loop defined by lines \texttt{5-14} of Algorithm~\ref{alg:RLCT_main}. By construction, this loop runs for at most $B =  (d_y+\frac{1}{2})d_x$ steps. Fix an input $f \in \QQ[x,y]$. Let $I \leq B$ be the total number of iterations ran by the algorithm before exiting the while loop, for input $f(x,y)$. Let $f^{(0)}(x,y)=f(x,y)$. Let $(b^{(0)},p^{(0)})=(0,0)$ For $i\in[I]$, let $(b^{(i)},p^{(i)},m^{(i)})$ be the triplet computed in line \texttt{10}. Let $f^{(i+1)}(x,y)=f^{(i)}(x,y-b^{(i)}x^{p^{(i)}})$ be the change of variables in line \texttt{11}, and let $B^{(i+1)}=B^{(i)}-(p^{(i)}-p^{(i-1)})$. Observe that $B^{(i)} = B - p^{(i)}$. Finally, let $P^{(i)}(x) = \sum_{j=0}^i b^{(j)}x^{p^{(j)}}$. Furthermore, denote by $\Delta^{(i)}$ the main face of $\newton(f^{(i)})$.

\begin{lemma}\label{lem:non_terminating}
    Let $f(x,y) \in \QQ[x,y]$ be an input of Algorithm~\ref{alg:RLCT_main}. $P_f(x)\in\QQ[[x]]\setminus\QQ[x]$ if and only if there exists $N\in\NN_{>0}$, such that $f^{(N)}$ is not normalized and $B^{(N)}\leq 0$.
\end{lemma}

\begin{proof}
    Let $f(x,y) \in \QQ[x,y]$ be an un-normalized polynomial, such that $\deg_y f=d_y$ and $\deg_x f=d_x$. Let $B = (d_y + \frac{1}{2})d_x$. Suppose that such an $N$ exists. Let $\phi_1(x),\ldots,\phi_n(x)$ be the $y$-roots of strictly positive orders, whose initial coefficients are roots of $F_{\Delta^{(0)}}$. Let $m_1\geq\ldots\geq m_n$ be their multiplicities as Puiseux $y$-roots of $f(x,y)$. Then $n \leq \deg_z F_{\Delta^{(0)}} = \sum_{i=1}^n m_i \coloneq m$. Now let $m^{(0)}=m^{(0)}_1 \geq\ldots \geq m_{r^{(0)}}^{(0)}$, for $r^{(0)} \leq n$ be the multiplicities of the distinct roots of $F_{\Delta^{(0)}}(z)$. Likewise, $m=\sum_{j=1}^{r^{(0)}} m^{(0)}_j$. Now, since $B^{(N)}\leq0$, it must be that $\deg P^{(N)}(x) = p^{(N)} \geq  B$. By Proposition~\ref{prop:complexity}, there exist no two distinct $y$-roots $\phi_i(x)$ and $\phi_j(x)$ as above, such that $\trunc{\phi_i(x)}{{p^{(N)}}}=\trunc{\phi_j(x)}{{p^{(N)}}}$. In other words, all the $y$-roots are separated, and $F_{\Delta^{(N)}}(z)$ has $r^{(N)}=n$ distinct roots, each with multiplicity $m^{(N)}_i=m_i$, for $i\in [n]$. Since $f^{(N)}$ is not yet normalized, we have that $m_1 > \delta_{f^{(N)}}$. So $\phi_1(x) \in \puiseux{\bar{\QQ}}{x}$ is a $y$-root of $f(x,y)$ of multiplicity $m_1 > \delta_f$.
    Now, Lemma~\ref{lem:maxmult} implies that $m_1 > m_i$ for all $1 < i \leq n$. Therefore, let $g(x,y)$ be the minimal polynomial of $\phi_1(x)$. Then, $f(x,y) = g(x,y)^{m_1}h(x,y)$, for some polynomial $h(x,y)\in\QQ[x,y]$. By Lemma~\ref{lem:minkowski} and Lemma~\ref{lem:dilation}, we have that $\delta_g < 1$. With this geometric constraint, $\newton(g)$ must be spanned by a single facet $\Delta = ( \bm{p}_0,\bm{p}_1)$, where $\bm{p}_0 = (0,1)$ and $\bm{p}_1=(a,0)$, for some $a\in\NN$. Then, $g(x,y)$ has a single $y$-root of strictly positive order $\phi_1(x)$. By \cite[Theorem 1]{banderier2013coefficients}, we have that $\phi_1(x) \in \QQ[[x]]\setminus\QQ[x]$. Moreover, the generalized multiplicity $m_1$ of $P_f(x)$ is greater than $\delta_f$. Therefore $P_f(x) = \phi_1(x)$. Now suppose that $P_f(x) \in \QQ[[x]]\setminus\QQ[x]$. Then, for all $i\in\NN_{>0}$, we have that $m^{(i)}>\delta_{f^{(i)}}$, and therefore $f^{(i)}$ is not normalized. 
\end{proof}

\subsection{Proof of Theorem~\ref{thm:correctness}}
\label{subsec:proof_thm:correctness}

\begin{proof}
    Suppose that $P_f(x) \in \QQ[x]$. By Lemma~\ref{lem:non_terminating}, there exists $N \in \NN_{>0}$ such that $f^{(N)}$ is normalized and $B^{(N)} > 0$. By Proposition~\ref{prop:unique}, $P^{(N)}(x)$ is the normalizing power series of $f(x,y)$, and the algorithm returns the correct output. Suppose that $P_f(x) \in \QQ[[x]]\setminus\QQ[x]$. By Corollary~\ref{cor:multiplicity_and_rlct}, $\rlct_0(f) = \frac{1}{m}$. Then, for any finite truncation $P(x)$ of $P_f(x)$, $\tilde{f}(x,y) = f(x,y-P(x))$ is not normalized. Therefore, Algorithm~\ref{alg:RLCT_main} exits the while loop defined by lines \texttt{5-14} and correctly outputs $\rlct_0(f)$.
\end{proof}

\subsection{Proof of Theorem~\ref{thm:arithmetic_complexity}}\label{subsec:arithmetic_complexity}

\begin{proof}
Observe that since $p^{(i)} = O(d_xd_y)$, we have $\deg_y f^{(i)} = d_y$ and $\deg_x f^{(i)} = O(d_xd_y^2)$, for all $i\in[I]$. The building blocks of Algorithm~\ref{alg:RLCT_main} are: Newton polygon computation, univariate square-free factorization, and change of variables computations. We address each one of them successively. Since computing the Newton polygon of $f^{(i)}$ is an operation on the coordinates of the points in the support of $f^{(i)}$ and not on its coefficients, its arithmetic cost is negligible. \\
\textbf{Univariate square-free factorization.} Noting that $\deg F_{\Delta^{(i)}}(z) \leq \deg_y f^{(i)} = d_y$ and using Yun's algorithm~\cite{yun1976square}, the square-free factorization step costs $O(M(d_y)\log(d_y)) = \tilde{O}(d_y)$. \\
\textbf{Change of variable.} Each change of variable $f^{(i+1)}(x,y)=f^{(i)}(x,y-b^{(i)}x^{p^{(i)}})$ is a monomial Taylor shift of polynomials $R[y]$ in the ring $R = \QQ[x]$, see \cite[Section 4.1]{gerhard2004modular}. By \cite[Theorem 4.5]{gerhard2004modular}, a fast Taylor shift in $R$ costs $O(M(D))$ additions and multiplication and $O(D)$ divisions, using the convolution method, where $D$ is the degree of the original polynomial. In our case, $R=\QQ[x]$, $D=O(d_y)$, and the divisions are divisions of polynomials in $R$ by scalars in $\QQ$. A multiplication in $\QQ[x]$ costs $M(D)$ in $\QQ$, where $D=O(d_x d_y^2)$ is the degree of the elements involved. Division of polynomials by scalars costs $O(d_x d_y^2)$. Putting this together, we get that the arithmetic complexity of a single change of variable  is \[O(M(d_y))\cdot O(M(d_x d_y^2)) +O(d_y)\cdot O(d_x d_y^2) = \tilde{O}(d_xd_y^3).\] \\
Since the arithmetic complexity of the change of variables dominates that of univariate square-free factorization and since the maximal number of steps is $O(d_xd_y)$, by Corollary~\ref{cor:rlctiters}, the total arithmetic complexity of the algorithm is $\tilde{O}(d_x^2d_y^4)$.

\end{proof}

\subsection{Proof of Theorem~\ref{thm:bit_complexity}}

To simplify the bit complexity analysis of Algorithm~\ref{alg:RLCT_main}, we may consider that the input polynomial $f(x,y)$ is primitive and has integer coefficients. We can always do this by multiplying $f(x,y) \in \QQ[x,y]$ by a scalar $c\in \QQ$. Scalar multiplication does not affect the real log canonical threshold, i.e. $\text{RLCT}_0 (f) = \text{RLCT}_0 (cf)$.

Before stating the proof of Theorem~\ref{thm:bit_complexity}, we recall some basic definitions around bitsizes of rationals and polynomials. Let $a\in\ZZ_{>0}$. The height of $a$ is defined as $h(a) \coloneqq \abs{a}$ and its bitsize is defined as the logarithm of the height, denoted by $\tau(a) \coloneq \log \abs{a}$. Now, the height and the bitsize of an integer polynomial $g(\bm{x}) = \sum_{0\leq\abs{\alpha}\leq d} g_\alpha \bm{x}^\alpha\in\ZZ[\bm{x}]$, are defined as the maximal height and bitsize attained by the coefficients of $f$, i.e. $h(f) \coloneqq \max_{0\leq\abs{\alpha}\leq d} g_\alpha $ and $\tau(g) = \log h(g)$. We now extend those definitions to rational numbers and rational polynomials. Let $q=\frac{u}{v}\in\QQ$ such that $\gcd(u,v)=1$. The height of $q$ is defined as $h(q) = \max(\tau(u),\tau(v))$, and the bitsize of $q$ as $\tau(q)\coloneq \log h(q)$. Given a rational polynomial $f(\bm{x}) = \sum_{0\leq\abs{\alpha}\leq d} f_\alpha \bm{x}^\alpha\in\QQ[\bm{x}]$, we define its \textit{denominator} as $v_f \coloneqq \text{denom}(f) \coloneqq \min\{ v \in \NN \mid vf\in\ZZ[x,y]\}$. Then the height and bitsize of $f(\bm{x})$ are defined as $h(f)\coloneqq \max(h(v_f),h(v_f f))$ and $\tau(f) \coloneqq \log h(f)$ respectively. 
To get a bound on the bit complexity of Algorithm~\ref{alg:RLCT_main}, it will be necessary to get an estimate on the bitsize of the coefficients $b_k$ of $P^{\text{fin}}_f(x) = \sum_{k=1}^d b_k x^k$, in terms of the size of the input $f(x,y)=\sum_{ij} c_{ij}x^iy^j\in\QQ[x,y]$. To this end, we let $d_x = \deg_x f$, $d_y = \deg_y f$, $h = h(f)$ and $\tau = \tau(f)$. In \cite{walsh2000polynomial}, the author derives a polynomial bit complexity bound on the computation of the singular part of Puiseux series \ref{def:singpart}. We will follow their analysis, with the added subtlety that we need only work with rational $b_k$'s as a full singular part computation might include algebraic numbers: this will improve the bit complexity bounds in the particular case of our algorithm. We first recall \cite[Theorem A]{walsh2000polynomial}, which follows from \cite{dwork1992eisenstein}'s quantitative estimate of the Eisenstein constant.}

\begin{proposition} [\cite{walsh2000polynomial},Theorem A]\label{prop:eisenstein_bound}
Let $G \in \mathbb{Z}[x,y]$ be a nonzero polynomial which has no multiple factors
when regarded as a polynomial in $y$. Let $M = \deg_x G$, $N = \deg_y G$, and
$H = \operatorname{ht}(G)$. If the formal power series
$y(x) = \sum_{k=0}^{\infty} b_k x^k$ satisfies $G(x, y(x)) = 0$, then there is an
integer $B$ with
\[
  B < B_G \coloneqq   4.8\left(8e^{-3} N^{4 + 2.74\log N} e^{1.22N} H^2 (1+M)^2\right)^N,
\]
for which $B^{M+k} b_k$ is an algebraic integer for all $k \geq 0$.
\end{proposition}

We will need an additional bound, due to \cite{schmidt1990eisenstein}, which we adapt from \cite{walsh2000polynomial}.

\begin{proposition}[\cite{walsh2000polynomial}]\label{prop:house_bound}
Let $F(x,y) \in \QQ[x,y]$. Let $P^{\text{fin}}_F(x) = \sum_{k=1}^d b_k x^d$ be the finite part of the normalizing power series of $F(x,y)$. Let $n = \deg_y F$, $m=\deg_x F$ and $h = h(F)$ Then,
\[
\abs{b_k} \leq  A_F \coloneqq 2(h+1)(h(m+1)(n+1))^{6mn^2},
\]
for all $1\leq k \leq d$.
\end{proposition}

We use the above propositions to prove the following lemma, which bounds the bitsize of the coefficients $b_k$ of $P^{\text{fin}}_f(x)$:

\begin{lemma}\label{lem:b_bound} 
    Let $f(x,y) \in \ZZ[x,y]$. Let $P^{\text{fin}}_f(x) = \sum_{k=1}^d b_kx^k$ be the finite part of its normalizing power series. Then $\tau(b_k) =O(d_y^3d_x(\tau+d_x+d_y)) $ for all $k\in[d]$.
\end{lemma}

\begin{proof} To use Proposition~\ref{prop:eisenstein_bound}, we cannot work with $f(x,y)$ directly, as $f(x,y)$ may not be square-free in $y$. Moreover, the Puiseux root $y(x)$ for which $P^{\text{fin}}_f(x) = \sum_{k=1}^d b_k x^k$ is the initial part might not be a power series, but a Puiseux series of ramification index $\tilde{n}>1$. To address these issues, we instead work with the polynomial $G = g(x^{\tilde{n}},y) \in \ZZ[x,y]$, where $g$ is a square free polynomial in $y$ dividing $f$, and for which $y(x)$ is a root. Note that $g\in\ZZ[x,y]$ is primitive by Gauss's Lemma. Then, $y(x^{\tilde{n}})$ is an algebraic power series, and we can use Theorem~\ref{prop:eisenstein_bound}. Let $M=\deg_x G \leq \tilde{n} d_x \leq d_y d_x$ and $N=\deg_y G \leq d_y$. Let $H = h(G) = h(g)
$. By a classical result of Mahler on the size of the coefficients of polynomial factors, see for example~\cite[Chapter 3, Proposition 2.12]{lang2013fundamentals}, $h(g) \leq e^{d_x+d_y}h$. Then, $H \leq  e^{(d_x+d_y)}h$. Let $T = \tau(G) \leq  \tau + d_x + d_y$.
Let $b_k = \frac{u_k}{v_k}$, such that $\gcd(u_k,v_k)=1$. Since $b_k\in \QQ$, by Proposition~\ref{prop:eisenstein_bound}, $B_G^{M+\tilde{n}k}b_{\tilde{n}k} \in \NN$, for all $1\leq k\leq d$. It follows that $v_k \mid B_G^{M+\tilde{n}d}$, for all $1\leq k \leq d$. Thus, 
\[
\abs{v_k} \leq  B_G^{M+\tilde{n}d}.
\]
Moreover, by Proposition~\ref{prop:house_bound}, \[\abs{u_k} = \abs{b_k}\abs{v_k} \leq A_GB_G^{M+\tilde{n}d}.\]
Therefore, \[\tau(b_k)\leq \log(A_GB_G^{M+\tilde{n}d}) = \log(A_G)+(M+\tilde{n}d)\log(B_G).\]

Passing to the logarithm, we have that 

\[
\log(A_G) = O(MN^2(T+\log(M+1)+\log(N+1))),
\]

and that

\[
\log(B_G) = O(N^2 + NT+N\log(M+1)).
\]

Altogether, \[
\tau(b_k) = O(MN^2(T+\log(M+1)+\log(N+1))+\tilde{n}d(N^2 + NT+N\log(M+1))).
\]

Since $M=O(d_yd_x)$, $N=O(d_y)$, $\tilde{n}d = O(d_y^2d_x)$ and $T=O(\tau +d_y+d_x)$, we obtain

\[
\tau(b_k) =O(d_y^3d_x(\tau+d_x+d_y)).
\]

\end{proof}
We can now state the proof of Theorem~\ref{thm:bit_complexity}.

\begin{proof}

In order to characterize the full bit complexity of Algorithm~\ref{alg:RLCT_main}, we will find a uniform upper bound on the bit size of the coefficients of the polynomials $f^{(i)}(x,y) = f(x,y-\sum_{k=1}^{i-1} b_k x^{p_k})$. For $p\in \NN_{>0}\cup \{\infty\}$, we define the norm of $f(x,y)=\sum_{ij}c_{ij}x^iy^j$ as $\lVert f \rVert_p =\left( \sum_{ij} \abs{c_{ij}}^p \right)^{\frac{1}{p}} $. Note that for $f\in\ZZ[x,y]$, the infinity norm of $f$ coincides with the height, that is $\lVert f \rVert_{\infty} = h(f)$. Note also that the $1$-norm is sub-multiplicative since $\lVert p(\bm{x})q(\bm{x})\lVert_1 \leq \lVert p(\bm{x})\rVert_1\lVert q(\bm{x})\rVert_1$, by Young's convolution inequality. Since the $f^{(i)}$ are rational polynomials, due to the $b_k$ terms, we will need to find an upper bound on $\text{denom}(f^{(i)})=v_{f^{(i)}}$ and on $h\left(v_{f^{(i)}} f^{(i)}\right)$. Let $V_i=\text{lcm}(v_1,\ldots,v_{i-1})$. Then, $V_i^{d_y}f^{(i)} \in \ZZ[x,y]$ and therefore $V_i^{d_y} \geq v_{f^{(i)}}$. Write $g^{(i)} = V_i^{d_y}f^{(i)}$, and write $\tau_b$ for the bound on the coefficients of $P^{\text{fin}}_f(x)$ from Lemma~\ref{lem:b_bound}. We will look for an upper bound on $h(g^{(i)})=\lVert g^{(i)} \rVert_\infty$. We get,

\begin{equation*}
    \begin{split}
        \lVert g^{(i)}\rVert_\infty &\leq \lVert g^{(i)}\rVert_1 \\ &\leq V_i^{d_y} \sum_{kl} \lVert c_{kl} x^k(y-P^{(i)}(x))^l \rVert_1 \\
        & \leq V_i^{d_y} \sum_{kl} \abs{c_{kl}}\lVert (y-P^{(i)}(x))^{d_y} \rVert_1 \\
        & \leq V_i^{d_y} \lVert f\rVert_1 \lVert y-P^{(i)}(x) \rVert_1^{d_y}\\
        &\leq  V_i^{d_y} \cdot (d_y+1)(d_x+1)h(f) \cdot ((d_y+\frac{1}{2})d_x e^{\tau_b}+1)^{d_y},
    \end{split}
\end{equation*}

where we used the submultiplicativity of $\lVert \cdot \rVert_1$,  and the degree bound $\deg P^{\text{fin}}_f(x) \leq ((d_y+\frac{1}{2})d_x$. As in the proof of Theorem~\ref{thm:arithmetic_complexity}, and using Proposition~\ref{prop:eisenstein_bound}, we have that $V_i \leq B_{f}^{M+\tilde{n}d}$, where $M=O(d_yd_x)$ and $\tilde{n}d = O(d_y^2d_x)$. Thus, $\log(V_i)\leq \tau_b$.

Passing to the logarithm,
\begin{equation*} \label{eq:bitsize_bound_fi}
        \begin{split}
        \tau(f^{(i)})&\leq 
\log \lVert g^{(i)} \rVert_{\infty} \\&\leq  d_y \tau_b +\log\left((d_y+1))(d_x+1)\right) + \tau + d_y\left(\log\left(\left(d_y+\frac{1}{2}\right)d_x e^{\tau_b} + 1\right)\right) \\
&= O(d_y^4d_x(\tau+d_x+d_y)),
        \end{split}
\end{equation*}

for all $i\in[I]$. All operations in Algorithm~\ref{alg:RLCT_main} are on the coefficients of the $f^{(i)}$'s. We can thus get the bit complexity of Algorithm~\ref{alg:RLCT_main} by multiplying the arithmetic complexity of Theorem~\ref{thm:arithmetic_complexity} with our prevously obtained coefficients bitsize bound:

\[ \tilde{O}(d_y^4d_x^2)\cdot O(d_y^4d_x(\tau+d_x+d_y)) = \tilde{O}(d_y^8d_x^3(\tau+d_x+d_y)).\]

\end{proof}
 \color{black}

\section{Auxiliary Algorithms}\label{app:aux_alg}

Here, we present the \texttt{Normalized} algorithm (Algorithm~\ref{alg:normalized}), a necessary sub-routine of Algorithm~\ref{alg:RLCT_main}. This algorithm takes as input a polynomial $f(x,y) \in \QQ[x,y]$, and returns \texttt{True} if the polynomial is normalized or a triple of integers $(b,p,m)$, where $b$ is the root of maximal multiplicity $m$ of the facet polynomial $F_\Delta$ of the main face $\Delta$, and $(p,q)$ is the weight of $\Delta$. This triple is then used to define the change of variables in Algorithm~\ref{alg:RLCT_main}.

\begin{algorithm}[ht]
\caption{\texttt{Normalized}}\label{alg:normalized}
\LinesNumbered
\KwIn{A polynomial $f(x,y)\in\mathbb{Q}[x,y]$, its Newton polygon $\newton(f)$, and its Newton distance $\delta_f$.}
\KwOut{True if the main facet $\Delta$ of $\newton(f)$ is normalized; otherwise a tuple $(b,p,m)\in\mathbb{Q}\times\mathbb{N}\times\mathbb{N}$.}

Identify the main face $\Delta$ of $\newton(f)$\;

Compute $F_{\Delta}(z)$\;

\eIf{$F_{\Delta}(z)\text{ is a single term}$}{
    \Return True\;
    }{
    Compute the weight $wt=(p,q)$ of $\Delta$\;
    
    \eIf{$q\neq 1$}{
        \Return True\;
        }{
        Compute the square-free factorization $F_{\Delta}(z)=\prod_{i=1}^m (F_{\Delta,i}(z))^i$\;
        
        \eIf{$m \le \delta_f$}{
            \Return True\;
            }{
            Let $F_{\Delta,m}(z)=(z-b)$\;
            
            \Return{$(b,p,m)$}\;
            }
        }
    }
\end{algorithm}

Note that since we know that $P_f(x)\in\QQ[[x]]$ by Lemma~\ref{lem:rational}, it is sufficient to compute a square-free factorization of $F_{\Delta_i}(z) \in \QQ$, as opposed to a full factorization in $\bar{\QQ}$.

\end{document}